\pdfoutput=1

\documentclass[journal, twocolumn]{IEEEtran}
\usepackage{amsmath,amsfonts}
\usepackage{algorithmic}

\usepackage{algorithm}
\usepackage{multirow}
\usepackage[numbers,sort]{natbib}
\usepackage{amsthm}
\usepackage{amssymb}
\usepackage{subfigure}
\usepackage[center]{caption}

\usepackage{lipsum}

\usepackage{array}
\usepackage[caption=false,font=normalsize,labelfont=sf,textfont=sf]{subfig}
\usepackage{textcomp}
\usepackage{stfloats}
\usepackage{url}
\usepackage{verbatim}
\usepackage{graphicx}

\def\BibTeX{{\rm B\kern-.05em{\sc i\kern-.025em b}\kern-.08em
T\kern-.1667em\lower.7ex\hbox{E}\kern-.125emX}}
\usepackage{balance}

\def\b#1{\mathbf{#1}}
\def\t#1{\textbf{#1}}
\def\r#1{\mathrm{#1}}
\def\trans{^{\r{T}}}
\def\eqta{\begin{equation}}
\def\eqtb{\end{equation}}
\def\alna{\begin{aligned}}
\def\alnb{\end{aligned}}
\def\arra{\begin{array}}
\def\arrb{\end{array}}
\IEEEoverridecommandlockouts
\begin{document}

\markboth{Journal of \LaTeX\ Class Files,~Vol.~18, No.~9, September~2020} 
{How to Use the IEEEtran \LaTeX \ Templates}

\title{Fast Semi-supervised Learning on Large Graphs: An Improved Green-function Method}
\author{Feiping Nie, \IEEEmembership{Senior Member, IEEE}, Yitao Song, Wei Chang, Rong Wang, and Xuelong Li, \IEEEmembership{Fellow, IEEE}

\thanks{The authors are are with the School of Artificial Intelligence, OPtics and ElectroNics (iOPEN), School of Computer Science, Northwestern Polytechnical University, Xi’an 710072, P.R. China, and also with the Key Laboratory of Intelligent Interaction and Applications (Northwestern Polytechnical University), Ministry of Industry and Information Technology, Xi’an 710072, P.R. China. (email: feipingnie@gmail.com; tombacsong@outlook.com; hsomewei@gmail.com; wangrong07@tsinghua.org.cn; li@nwpu.edu.cn).}}

\markboth{Journal of \LaTeX\ Class Files,~Vol.~18, No.~9, September~2021}%
{How to Use the IEEEtran \LaTeX \ Templates}

\maketitle 

\begin{abstract}
In the graph-based semi-supervised learning, the Green-function method is a classical method that works by computing the Green’s function in the graph space. However, when applied to large graphs, especially those sparse ones, this method performs unstably and unsatisfactorily. We make a detailed analysis on it and propose a novel method from the perspective of optimization. On fully connected graphs, the method is equivalent to the Green-function method and can be seen as another interpretation with physical meanings, while on non-fully connected graphs, it helps to explain why the Green-function method causes a mess on large sparse graphs. To solve this dilemma, we propose a workable approach to improve our proposed method. Unlike the original method, our improved method can also apply two accelerating techniques, Gaussian Elimination, and Anchored Graphs to become more efficient on large graphs. Finally, the extensive experiments prove our conclusions and the efficiency, accuracy, and stability of our improved Green’s function method.
\end{abstract}


\maketitle

\section{Introduction}
\IEEEPARstart{A}{s} more and more information is generated and collected, more and more data needs to be classified and is utilized in various tasks. Machine learning models bring convenience to us, but they usually require labels to deal with the data for higher performance. Although the volume of data is rapidly growing, it’s hard for labeled data to grow at the same pace\cite{Gao2016SemiSupervisedSR}. The economic and human costs of data labeling are too high, which leads to many classification problems with only a few samples labeled and most samples unlabeled.
Semi-supervised learning\cite{Nie2010FlexibleME} addresses this problem by trying to classify samples with a small number of labeled ones and a large number of unlabeled ones. Semi-supervised learning still classifies under supervision but learns from unsupervised learning methods. Like supervised learning, semi-supervised learning utilizes the label information of labeled samples and hopes that the classification results of labeled ones are consistent with their labels. The difference is a basic assumption named smoothness or clustering that is popular both in semi-supervised and unsupervised learning: The samples that are close to each other are more likely to share a label\cite{Chapelle2002ClusterKF, Zhou2003LearningWL, Zhou2018ABI}. With this assumption, semi-supervised learning can utilize the location of unlabeled data except for those of labeled data and achieve results from the positional relationship between them.

As a part of semi-supervised learning, transductive learning\cite{xiaojin2002learning, Wang2006LabelPT} involves all the test data as unlabeled samples and takes advantage of their features. It is believed that taking all classifying targets into account results in better performance and less run time of the algorithm, which has been verified in methods such as learning with local and global consistency(LLGC)\cite{Zhou2003LearningWL, gu2014combining}, Gaussian fields and harmonic function(HF)\cite{Zhu2003SemiSupervisedLU}, the recent special label propagation(SLP)\cite{Nie2010AGG}, etc. Because of the known distribution of all the training and testing data, graph theory measures the relationship between any two samples in the sample space which is a reproducing kernel Hilbert space\cite{belkin2006manifold}. Graph-based semi-supervised learning (GSSL) methods\cite{Nie2010AGG, Zhu2003SemiSupervisedLU, Qiu2019AcceleratingFM} describe the positional relationship between samples as a graph\cite{Wang2006LabelPT, blum2001learning}. A graph in graph theory consists of points of samples connected by edges and can be described as a similarity matrix. In most cases, similarities are computed by $\epsilon$-neighborhood\cite{belkin2003laplacian}, local linear representation\cite{Wang2006LabelPT, saul2003think}, heat kernel\cite{belkin2004semi}, Gaussian kernel\cite{wang2017fast} or other metrics in the space. The Laplacian matrix also called the graph Laplacian, can be viewed as a matrix form of the discrete Laplace operator on a graph\cite{belkin2008towards}. The matrix plays an important role in GSSL. Many useful properties of the graph can be deduced by analyzing the Laplacian matrix. The Green-function method\cite{Ding2007ALF} is a good example and is proposed by designing Green’s function on a graph which can be seen as the inverse of the Laplace operator.

However, the explanation of the Green-function method is given as a whole and lacks physical meanings, which is widely different from other GSSL methods. Besides, the Green-function method performs badly on non-fully connected graphs and thus is hard to apply to large sparse graphs. In view of the above problems, we deduce a GSSL method from an optimization problem and prove its equivalence to the Green-function method on fully connected graphs, which can be seen as another interpretation. Through the proofs, we deduce the physical meaning of it. Then, we find the undesirable reason for using it on a non-fully connected graph and propose a workable approach to improve. To apply the Green-function method on large graphs, two accelerating techniques, Gauss Elimination, and Anchored Graphs are introduced to ease the high time and space complexity. Our main contributions in this paper are listed as follows:

\setlength{\hangindent}{2.8em}
\setlength{\hangafter}{1}
1)\,\, First, we give a novel interpretation of the Green-function method on fully connected graphs. Through the proofs, we deduce the physical meaning. Second, we analyze why we shouldn't use it on a non-fully connected graph, and then use the perturbation strategy to improve it. During the argumentation, several interesting conclusions are drawn.

\setlength{\hangindent}{2.8em}
\setlength{\hangafter}{1}
2)\,\, We introduce two accelerating techniques, Gauss Elimination, and Anchored Graphs into our improved Green-function method for large graphs. The method using Gauss Elimination has a smaller coefficient of complexity but has the same result. The Green-function method on Anchored Graphs performs better in most cases and has the time complexity of $O(nm^2)$ and space complexity of $O(nm)$, where $n$ denotes the number of samples and $m$ denotes the number of anchor points.

\setlength{\hangindent}{2.8em}
\setlength{\hangafter}{1}
3)\,\,\, Various experiments are conducted to prove our conclusions, and they validate the good performance and high efficiency of our proposed methods compared with the original Green-function method and others.

Fig. \ref{fig:Contri} is a schematic diagram of our contributions.
GSSL, GF, and LLGC are introduced in Section \ref{sec:relate}.
Theorem 1\&2 and Conclusion 1-3 are drawn in Section \ref{sec:method}.
In Section \ref{sec:acclr}, we give out the procedures of two accelerating techniques in detail.
In Section \ref{sec:Compare}, we discuss the relationship between our proposed method and LLGC.

\begin{figure}[h]
\centering
\includegraphics[width=3.5in]{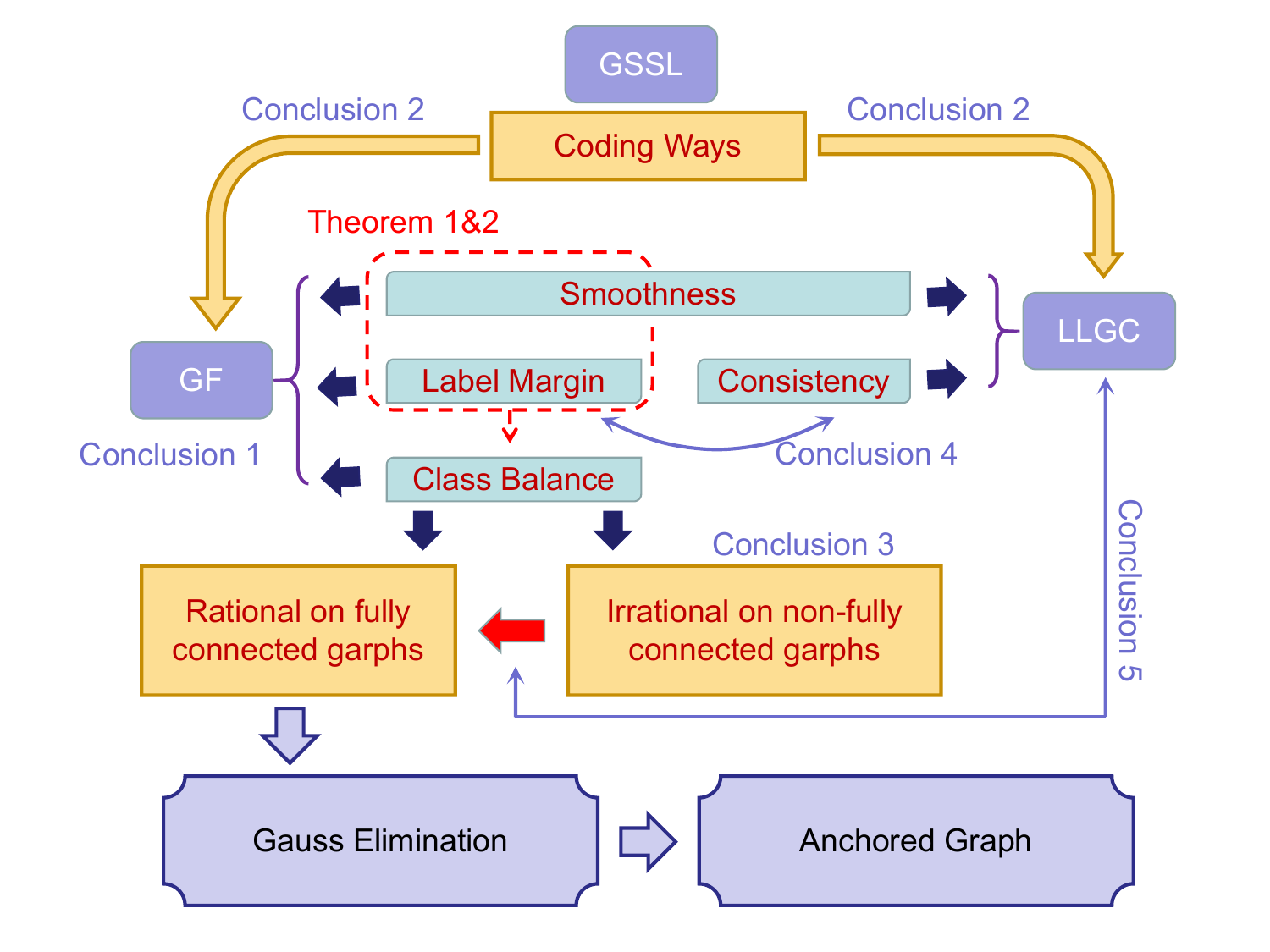}
\caption{The schematic diagram of the paper.}
\label{fig:Contri}
\end{figure}

\section{Related Work} \label{sec:relate}
\subsection{Overview of Graph-based Semi-supervised Learning}
\noindent Graph-based Semi-supervised learning is to classify all the samples when a few are labeled while more samples’ labels are unknown. Usually, we have the features of all samples so as to calculate the similarity between any two of them.

Assume that the first $l$ samples of the given data set $\{\b x_i\}_{i=1}^n$ are labeled into several classes, and their labels form a set $\{y_i\}_{i=1}^l$ where $\b x_i \in \mathcal X$ denotes the $i$-th sample, $y_i \in \mathbb Z_+$ denotes which class the $i$-th sample is in, and $n$ denotes the amount of all.
For the sake of calculation, we use a matrix $\b Y \in \mathbb R^{n \times c} $ to describe the situation of labels where $c$ means the number of classes. We could design two kinds of label matrixes $\b Y^{(1)}$ and $\b Y^{(2)}$ which differ on the labeled negative samples.

For each class, if we set the value of labeled negative samples as $-1$, the element of the label matrix $\b Y^{(1)}$ can be:
\eqta\label{eq:y1}
Y^{(1)}_{ik} = \left\{
\alna
+1&, &&i \le l, & k =y_i \\
-1&, &&i \le l, & k \neq y_i \\
0&, &&l < i \le n&
\alnb
\right.,
\eqtb
where $Y^{(1)}_{ik}$ means the element of $\b Y^{(1)}$ in the $i$-th row, the $j$-th column.

Or like the one-hot code, the value of labeled negative samples is $0$. We get
\eqta\label{eq:y2}
Y^{(2)}_{ik} = \left\{
\alna
1 &,&&i \le l \;\r{and}\; k = y_i\\
0 &,&&\r{otherwise}&
\alnb
\right..
\eqtb

Then we get two kinds of label matrix $\b Y = (\b y_1, \b y_2...,\b y_n) \trans \in \mathbb R ^{n \times c}$ where the elements in the $(l+1)$-st to $n$-th rows of $\b Y$ is all zeros, and the 1st to $l$-th rows of $\b Y$ describe labeled samples in two different coding ways. In this paper, we mainly use $\b Y^{(1)}$ as the label matrix $\b Y$, and in Section \ref{sec:Equivalence}, the equivalence between the two coding ways in our method is proved.

Supposing the sample space $\mathcal X$ is a metric space, thus similarities between every two samples can be measured. Define the similarity matrix $\b S \in [0,1] ^ {n \times n}$ where its element $S_{ij} \in [0, 1]$ describes the similarity between $\b x_i$ and $\b x_j$. $S_{ij} = 0$ means the unrelated pair, while $S_{ij} = 1$ means the exactly alike pair. Due to the symmetry of the relationship, $\b S$ is a symmetric matrix, and $S_{ij} = S_{ji}$. The degree vector $\b d$ describes the importance of each point on the graph and satisfies $\b{d} = \b S \b 1_n$, where the 1-vector with $n$ elements is $\b 1_n =(1,...,1)\trans$. We denote the diagonal matrix $\b D = \r{Diag} (\b d)$ as the degree matrix.

Solving a transductive learning problem means finding a soft label matrix $\b F = (\b f_1, \b f_1...,\b f_n) \trans \in \mathbb{R} ^ {n \times c}$, where the $j$-th element of column vector $\b f_i \in \mathbb{R}^c$ measures the tendency that the $i$-th sample is in the $j$-th class. Then, the predicted label of $i$-th sample will be
\eqta\label{eq:y_result}
\hat {y_i} = \arg \max_j (\b f_i)_j .
\eqtb

\subsection{Green-function Method} \label{sec:GF}
\noindent C. Ding \textit{et al.} \cite{Ding2007ALF} proposed the Green-function method by an analogy of Laplace operator $\mathcal L$ and Green’s function $\mathcal G(\b x, \b x_0)$ between Euclidean space and the graph. In Euclidean space, the Laplace operator of the function $f(\b x)$ can be written as
$$
\mathcal L f(\b x) = \nabla ^2 f(\b x) = \left( \sum \frac{\partial ^2}{\partial x_i ^2} \right) f(\b x).
$$

With Green’s Function $\mathcal G(\b x, \b x_0)$ in Euclidean space\cite{greenberg2015applications}, a linear and inhomogeneous equation like $\mathcal L f(\b x) = y(\b x)$ can be solved as
$$
f(\b x) = \mathcal L ^ {-1} y(\b x) \equiv \int_\Omega \mathcal G(\b x - \b x_0) y(\b x_0) d \b x_0 .
$$

It is well-known that Green’s function $\mathcal G(\b x, \b x_0)$ satisfies
$$
\mathcal L_{\b x} \mathcal G(\b x, \b x_0) = \delta(\b x - \b x_0),
$$
where $\mathcal L_{\b x}$ denotes the Laplace operator acting on $\b x$ and $\delta(\b x - \b x_0)$ denotes the Impulse function. These properties make Green’s function an important tool in solving the initial value or boundary condition problems of inhomogeneous differential equations.

When it comes to a graph, the function $f(x)$ is expressed as a vector $\b f$ and the Laplace operator performs as the left multiplying by the Laplacian matrix $\b L = \b D - \b S$, that is, $\mathcal L \b f = \b L \b f$.
Bearing the property of the Laplace operator, there are $\b L \b f = \b y$ and $\b f = \b G \b y$. The Green’s function $\b G$ is expected to satisfy $\b L \b G = \b I$ where $\b I$ denotes the identity matrix.

However, the singularity of $\b L$ means the in-existence of the inverse matrix. To bypass the trouble, the Green-function method discards the zero-mode of $\b L$. It is easy to construct a set of orthogonal eigenvectors $\{ \b u_1, \b u_2, ..., \b u_n \}$ of $\b L$ and their corresponding eigenvalues $\{ \sigma_1, \sigma_2, ..., \sigma_n \}$ where $\sigma_1 \leq \sigma_2 \leq ... \leq \sigma_n$ and norm of each eigenvector equals to 1:
\eqta\label{eq:L_svd}
\alna
&\b L \b u_k = \sigma_k \b u_k, ||\b u_k||_2^2 = 1, k=1,2,...,n\\
&\b u_p \trans \b u_q = 0, p \neq q, p,q=1,2,...,n
\alnb
\eqtb

In a fully connected graph, we get $0 = \sigma_1 < \sigma_2 \leq ... \leq \sigma_n$ according to the theory in \cite{spielman2012spectral}. With zero-mode discarded, it can be obtained that
$$
\b G = \b L^{\dagger} = \sum\limits_{i = 2}^n \frac{\b u_i {\b u_i}\trans }{\sigma_i}.
$$

Then the soft label matrix can be written as
\eqta\label{eq:F_GF}
\b F = \b G \b Y = \b L^\dagger \b Y,
\eqtb
where $\b L^\dagger$ is the Moore-Penrose inverse for $\b L$.

The method is also explained with electric resistor networks or random walks\cite{Ding2007ALF, Wang2008SemisupervisedLB}. These assumptions explain the Green-function method as a whole, but can not tell us in more detail what happens when using the method. A novel interpretation is shown in Section \ref{sec:Derivation}.

GF is used as a shorthand for the method.

\subsection{Learning With Local and Global Consistency} \label{llgc}
\noindent D. Zhou \textit{et al.}\cite{Zhou2003LearningWL} proposed the method of Learning with Local and Global Consistency(LLGC), which is to minimize the cost function:
$$
\mathcal Q(F) = \frac{1}{2} \left( \sum_{i,j=1}^n S_{ij} || \frac{\b f_i}{\sqrt D_{ii}} - \frac{\b f_j}{\sqrt{D_{jj}}}||^2_2 + \gamma \sum_{i=1}^n ||\b f_i - \b y_i ||^2_2\right)
$$
where $\gamma > 0$ is the regularization parameter, $S_{ij}$ means the similarity between the $i$-th and $j$-th samples and $D_{ii}$ means the degree of the $i$-th sample.

For graphs in which the similarity matrix $\b S$ is doubly-stochastic, the degree of any sample equals $1$ and the method is expressed as the optimization problem:
\eqta\label{eq:LLGC}
\min_{\b F} \sum_{i,j=1}^n S_{ij} || \b f_i - \b f_j ||^2_2 + \gamma \sum_{i=1}^n ||\b f_i - \b y_i ||^2_2
\eqtb

And the solution is
\eqta\label{eq:F_LLGC}
\b F = (\b L + \gamma \b I)^{-1} \b Y.
\eqtb

The cost function was explained with the \textit{smoothness constraint} and the \textit{fitting constraint}. The left-hand term named the \textit{smoothness constraint} forces the soft label vectors of nearby points to barely differ from each other, while the right-hand term named the \textit{fitting constraint} means that predicted labels are supposed to be close to the initial labels. For labeled samples, the soft label vectors are close to the truth of $1$ while for unlabeled samples, those are close to zero.

There is something counter-intuitive about the \textit{fitting constraint}, which is elaborated in Section \ref{sec:Compare}.

LLGC is used as a shorthand for the method.

\section{Proposed Method and Theoretical Derivation} \label{sec:method}
\noindent In this paper, the proposed model can be reformulated as the following form:
\eqta\label{eq:min}
\min_{\b F\trans \b 1_n = \b 0} \r{Tr}(\b F\trans \b L \b F) - 2 \gamma \r{Tr}(\b F\trans \b Y),
\eqtb
which can be seen as an interpretation of the Green-function method from the perspective of optimization.

As is shown in Fig. \ref{fig:Contri}, we propose two rules \textit{smoothness} and \textit{label margin} when analyzing the graph-based semi-supervised problem and summarize the third rule \textit{class balance} by proving the Theorem 1\&2. With the three rules, we obtain the method and prove the equivalence between it and the Green-function method on a fully connected graph, which means another interpretation. Therefore, the physical meanings of the Green-function method can be deduced as Conclusion 1.

For many GSSL methods like the Green-function method, the results with different coding ways are proved to be exactly equal to each other, which is summarized in Conclusion 2. After that, we draw Conclusion 3 of why the Green-function method performs worse on non-fully connected graphs and propose a workable approach to improve it.

\subsection{Rules for Derivation} \label{sec:Rules}
\noindent To deduce our proposed method, three rules are summarized. They’re \textit{smoothness}, \textit{label margin}, and \textit{class balance}.

\t{Smoothness Rule} is the same as that in LLGC\cite{Zhou2003LearningWL}, which means the more similar any two samples are, the shorter the distance between their soft label vectors is.

To make the graph smoother, the formula
$$
\sum_{i=1}^n S_{ij} ||\b f_i - \b f_j ||_2^2
$$
needs to be minimized.
In this way, a large $S_{ij}$ and a large difference between $\b f_i$ and $\b f_j$ will be of great cost.

\t{Label Margin Rule} performs differently on different types of samples. The labeled positive samples whose label $y$ is $+1$ have soft labels $f$ as large as possible, the labeled negative samples whose label $y$ is $-1$ have those as small as possible, and those of unlabeled samples are of no restrictions. This rule can make a great margin between positive and negative samples and is practiced by maximizing ${\b f_i}\trans \b y_i$ for each sample $\b x_i$. As a whole, it needs to maximize that
$$
\sum_{i=1}^n {\b f_i}\trans \b y_i.
$$

With the first two rules, we achieve the problem:
\eqta\label{eq:min_v_without}
\min_{\b F} \sum_{i=1}^n S_{ij} ||\b f_i - \b f_j ||_2^2 - 2 \gamma \sum_{i=1}^n {\b f_i}\trans \b y_i,
\eqtb
where $\gamma > 0$ is a parameter weighing the importance of the first two rules. In the subsequent derivation, its value won’t affect the classification results.

Problem (\ref{eq:min_v_without}) can be transformed into a matrix form as:
\eqta\label{eq:min_without}
\min_{\b F} \r{Tr}(\b F\trans \b L \b F) - 2 \gamma \r{Tr}(\b F\trans \b Y),
\eqtb
where $\b L = \b D - \b S$ is the Laplacian matrix.

Before giving up the third rule \textit{class balance}, we’d like to introduce several theorems to analyze the problem (\ref{eq:min_without}).

\subsection{Theorems About Solving the Problem (\ref{eq:min_without})} \label{sec:Proof}

\newtheorem{lemma}{Lemma}
\newtheorem{theorem}{Theorem}
\newtheorem{definition}{Notation}
\newtheorem{corollary}{Corollary}
\renewcommand{\qedsymbol}{$\blacksquare$}

The problem can be transformed into the format
$$
\min_{\b F} \sum_{i=1}^{n}(\b f_i\trans \b L \b f_i - 2 \gamma \b f_i\trans \b y_i).
$$

Therefore, it can be solved by solving the more common problem:
\eqta\label{eq:Ax}
\min_{\b x} \b x\trans \b A \b x - 2 \b x\trans \b b.
\eqtb

\vspace{1em}
\begin{definition}[]
Assume that $\b A \in \mathbb R^{n \times n}$ is a semi-positive definite matrix whose rank is $r$. $\b A$ can be decomposed compactly as $\b A = \b U \b \Sigma \b U\trans$, where $\b \Sigma \in \mathbb R ^{r \times r}$ is a diagonal matrix with non-zero eigenvalues as the diagonals and $\b U \in \mathbb R^{n \times r}$ is formed by the orthogonal eigenvectors corresponding to them. Denote $n-r$ as $h$, ${\b U_\perp \in \mathbb R^{n \times h}}$ as the orthogonal complement of $\b U$, and $\b A^\dagger$ as $\b U \b \Sigma^{-1} \b U\trans$. It is easy to validate that $\b A^\dagger$ is the Moore-Penrose inverse of $\b A$ with the definition.
\end{definition}

\vspace{1em}
\begin{lemma}[]
Let's say $\b x$ is an unknown $n$-dimensional vector. There are solutions to $\b A \b x = \b b$ if and only if ${\b U_\perp} \trans \b b = \b 0$, where $\b 0$ denotes a zero vector.
\end{lemma}

\begin{proof}
Union of $\b U$ and its orthogonal complement $\b U_\perp$ spans the entire space, so we could assume $\b b$ as $$\b b = \b U \b p + {\b U_\perp} \b q,$$ where $\b p$ and $\b q$ are both vectors.

Firstly, we’d like to prove that if ${\b U_\perp} \trans \b b \neq \b 0$, there’s no solution to $\b A \b x = \b b$.
With ${\b U_\perp} \trans \b b \neq \b 0$, we have $${\b U_\perp} \trans \b U \b p + {\b U_\perp}\trans {\b U_\perp} \b q \neq \b 0.$$

Because of ${\b U_\perp} \trans \b U = \b O$ and ${\b U_\perp}\trans {\b U_\perp} = \b I$, we know that $\b q \neq \b 0$.
However, if we substitute $\b b = \b U \b p + {\b U_\perp} \b q,$ and $\b A = \b U \b \Sigma \b U\trans$ into $\b A \b x = \b b$, it can be obtained that $$\b U (\b \Sigma \b U\trans \b x - \b p) = \b U_\perp \b q.$$

Since $\b U$ and $\b U_\perp$ are linearly independent, $\b q$ must be a zero vector which leads to a contradiction. There’s no solution as a result.

Secondly, we’ll give out a solution to $\b A \b x = \b b$ when ${\b U_\perp} \trans \b b = \b 0$.
With ${\b U_\perp} \trans \b b = \b 0$, we know that $\b q = \b 0$ and $\b b = \b U \b p$. It can be validated that $\b x = \b U \b \Sigma^{-1} \b p$ is a solution to $\b A \b x = \b b$.
\end{proof}

\vspace{1em}
\begin{theorem}[]
The minimization problem (\ref{eq:Ax}) has no solutions in the real space when ${\b U_\perp} \trans \b b \neq \b 0$, while has infinite solutions when ${\b U_\perp} \trans \b b = \b 0$.
\end{theorem}

\begin{proof}
The extreme value appears when the derivative equals zero, that is $\b A \b x - \b b = \b 0$. As is proved in Lemma 1, there’s no solution when ${\b U_\perp} \trans \b b \neq \b 0$, so does the extreme value. This problem will have minimization only when $\b x$ tends to infinity and that is a trivial solution.

When ${\b U_\perp} \trans \b b = \b 0$, $\b x$ can be expressed in terms of $\b U$ and $\b U_\perp$ as $$\b x = \b U \alpha + \b U_\perp \beta,$$ where $\alpha$ and $\beta$ are both vectors.

Plug it and $\b A = \b U \b \Sigma \b U\trans$ into $\b A \b x = \b b$, and we achieve $$\b U \b \Sigma \alpha = \b b \Rightarrow \alpha = \b \Sigma^{-1} \b U\trans \b b.$$

With $\alpha$ inferred, we get $$\b x = \b A^\dagger\b b + \b U_\perp \beta.$$

$\{\b A^\dagger\b b + \b U_\perp \beta: \beta \in \mathbb R^h \}$ is the set of solutions.

The minimum is achieved when and only when $\b x$ equals any element in the set, and it is obviously an infinite set.
\end{proof}

\vspace{1em}
\begin{corollary}[]
For any Laplacian matrix $\b L$ whose null space's orthogonal basis vectors from an matrix $\b U_\perp$, there are solutions to the minimization problem (\ref{eq:min_without}) if and only if $\b U_\perp\trans \b Y = \b O$ where $\b O$ denotes the zero matrix. And any element in the set $$\{\gamma \b L^\dagger\b Y + \b U_\perp \beta: \beta \in \mathbb R^{h \times c} \}$$ is a solution, where $h$ denotes the size of the null space, and $n$ and $c$ denotes the size of $\b F$ and $\b Y$.
\end{corollary}


According to the corollary, the classification $\hat y$ obtained by solving the problem (\ref{eq:min_without}) will vary when $\beta$ varies, which is unacceptable.
We try to find the situation when there’s only one solution and the solution is better to be the main part of the set.
Fortunately, we found it.

\vspace{1em}
\begin{theorem}[]
The minimization problem $$\min_{{\b U_\perp} \trans \b x = \b 0} \b x\trans \b A \b x - 2 \b x\trans \b b$$ has one and only one solution $\b x = \b A^\dagger \b b$.
\end{theorem}

\begin{proof}
Denote $\b x$ in terms of $\b U$ and $\b U_\perp$ as $$\b x = \b U \alpha + \b U_\perp \beta,$$ where $\alpha$ and $\beta$ are both vectors.

Due to the restriction ${\b U_\perp} \trans \b x = \b 0$, we obtain $${\b U_\perp} \trans \b U \alpha + {\b U_\perp}\trans {\b U_\perp} \beta = \b 0,$$.

Because of ${\b U_\perp} \trans \b U = \b O$ and ${\b U_\perp}\trans {\b U_\perp} = \b I$, we know that $\beta = \b 0$ and $\b x = \b U\alpha$.

The extreme value appears when the derivative equals zero, that is $\b A \b x - \b b = \b 0$. Plug $\b x = \b U\alpha$ into it, and we have $\alpha = \b \Sigma^{-1} \b U\trans \b b.$ The only possible solution is $$\b x = \b U \alpha = \b A^\dagger \b b.$$

And it can be proved that $\b x = \b A^\dagger \b b$ is exactly a solution to the problem.
\end{proof}

\vspace{1em}
\begin{corollary}[]
For any Laplacian matrix $\b L$ whose null space's orthogonal basis vectors from an matrix $\b U_\perp$, the minimization problem 
\eqta\label{eq:min_U}
\min_{\b F\trans \b U_\perp = \b O} \r{Tr}(\b F\trans \b L \b F) - 2 \gamma \r{Tr}(\b F\trans \b Y),
\eqtb
has and only has the solution $\b F = \b L^\dagger\b Y.$
\end{corollary}

According to Corollary 1 and 2, if we want to guarantee one and only one solution to the problem (\ref{eq:min_without}), the constraint $\b F\trans \b U_\perp = \b O$ is necessary. And the solution will be the main part of the infinite set of solutions that might exist.

\subsection{Our Proposed Method and its Physical Meanings} \label{sec:Derivation}

\noindent On a connected graph, the null space of the Laplacian matrix $\b L$ is one-dimensional and one of the basis vectors of the space is the 1-vector. In the other word, $\b U_\perp$ equals to the vector $\frac{1}{\sqrt n} \b 1$, and then, the problem (\ref{eq:min_U}) turns into the problem (\ref{eq:min}). The method we proposed is a solution to this problem.

With $\b F = \gamma \b L^\dagger \b Y$, the value $\gamma > 0$ won’t change the predicted result in Eq. (\ref{eq:y_result}), so we simply use $\gamma = 1$ latter in the paper and the solution is equivalent to the Green-function method in Eq. (\ref{eq:F_GF}).

The constraint $\b F\trans \b 1 = \b 0$ is explained as a new rule named \textit{class balance}:

\t{Class Balance Rule} ensures that there’s one and only one result to the problem. With this rule on a connected graph, it is guaranteed that any column of $\b F$ has an average of zero. The soft labels are evenly distributed on both sides of zero for each class in order that an intuitive $\b F$ is selected from the set of solutions.

With the three rules, we interpret the Green-function method in an intuitive and novel perspective and draw a conclusion.

\t{Conclusion 1.} On a connected graph, the Green-function method tends to sustain the smoothness of the whole graph, to make margins as large as possible between positive and negative samples for each class, and to make soft labels of all samples for each class to evenly distribute on both sides of 0.

\subsection{Equivalence Between Coding Ways} \label{sec:Equivalence}
\noindent To classify samples as Eq. (\ref{eq:y_result}), the initial labeling matrix $\b Y$ is of great importance. In this section, the coding ways as Eq. (\ref{eq:y1}) and as Eq. (\ref{eq:y2}) are proved to be equivalent. In another word, there is no difference in whether to set the negative samples at $0$ or at $-1$.

According to Eq. (\ref{eq:y1}) and  Eq. (\ref{eq:y2}), we have:
$$
\b Y^{(1)} = \left[ \arra{cc} 2 \b \Phi - \b 1_l {\b 1_c}\trans \\ \b O \arrb \right], \b Y^{(2)} = \left[ \arra{cc} \b \Phi \\ \b O \arrb \right]
$$
where $\b \Phi \in \mathbb R ^{l \times c}$ denotes the upper part of $\b Y^{(2)}$.

For the Green-function method and some other methods like LLGC, the solution can be calculated into the format of $\b F = \b G \b Y$, where $\b G$ is $\b L^\dagger$ for GF and $(\b L + \mu \b I)^{-1}$ for LLGC\cite{Zhou2003LearningWL}.

Then the two types of soft label matrixes are
$$\alna
\b F^{(2)} &= \b G \b Y^{(2)} \\
\b F^{(1)} &= \b G \b Y^{(1)} = 2 \b F^{(2)} - \b g {\b 1_c}\trans,
\alnb$$
where $\b g = \b G \left[ \arra{cc} \b 1_l \\ \b 0_u \arrb \right]$.

For the $i$-th sample, it is satisfied that $\b f^{(1)}_i = 2 \b f^{(2)}_i - (\b g)_i {\b 1}_c \trans$. Obviously, it will be allocated into the same class $j$ for both.

\t{Conclusion 2.} Whether to code the labeled negative samples as $-1$ or $0$ won’t affect the final classification result of GF and some other methods such as LLGC\cite{Zhou2003LearningWL}.

\subsection{When It Comes to Non-fully Connected Graphs} \label{sec:Non-fully}
\noindent Everything we’ve talked about so far has been on fully connected graphs, including the Moore-Penrose inverse of $\b L$ and the optimization problem we put forward. Some improvement needs to be introduced when it comes to non-fully connected graphs.

In a graph with $h$ connecting pieces, $\b L$ has a null space of $h$ dimensions and the smallest $h$ eigenvalues of $\b L$ equal zero. If we still use the Moore-Penrose inverse as the solution, Eq. (\ref{eq:F_GF}) turns into:
$$
\b F = \b G \b Y = \b L^\dagger \b Y = \left( \sum_{i = h+1}^n \frac{\b u_i {\b u_i}\trans}{\sigma_i} \right) \b Y .
$$

This corresponds to the constraint in the problem (\ref{eq:min_U}):
$$
\b F\trans \b U_\perp = \b O
$$
where $\b U_\perp = (\b u_1, \b u_2, ..., \b u_h)\in \mathbb R ^{n \times h}$ spans the null space of $\b L$. However, the constraint is unreasonable.

\begin{figure}[b]
\centering
\includegraphics[width=3.5in]{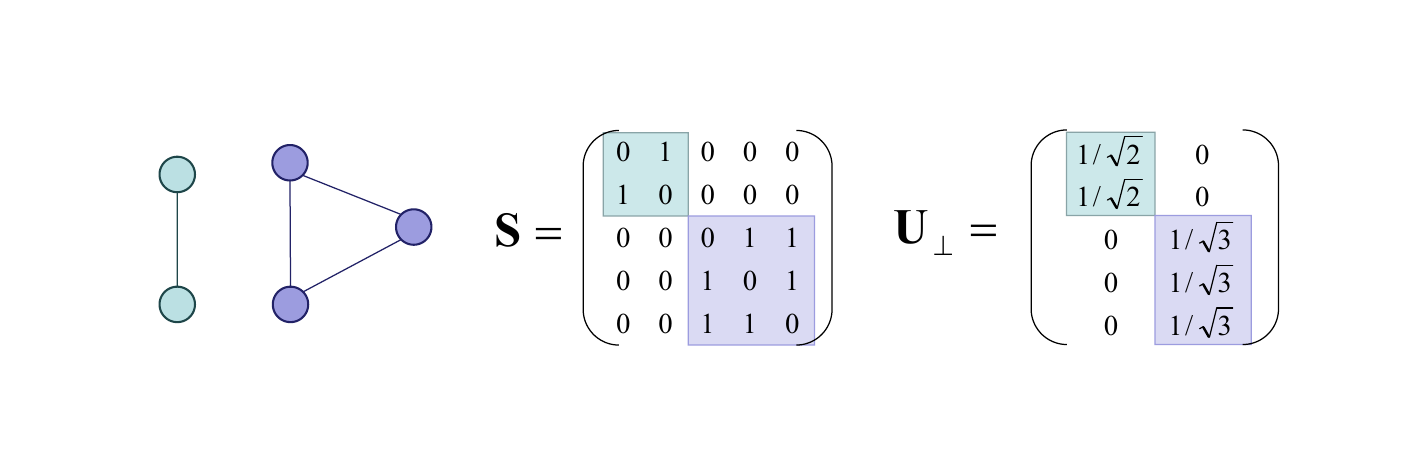}
\caption{An example of $\b U_\perp$ on 2-connected graph. It can be validated that $\b U_\perp$ spans the null space of $\b L$.}
\label{fig:h-connected}
\end{figure}

The explanation of the constraint can be based on any $\b U_\perp$ that fits the requirement. Let’s assume a special $\b U_\perp$. Each vector $\b u_i$ in it describes the situation of the $i$-th connected subgraph. If $\b u_i$ is corresponding to a connected subgraph with $n_i$ points, $n_i$ elements of $\b u_i$ equal $1 / \sqrt {n_i}$, and others are zero. It is obvious that the matrix fits the requirement and an example is given as Fig. \ref{fig:h-connected}.

As a result, the constraint turns into
$$
\b F\trans \begin{bmatrix}
\frac{1}{\sqrt{n_1}} \b 1_{n_1} & \b 0_{n_1} & \cdots & \b 0_{n_1}\\
\b 0_{n_2} & \frac{1}{\sqrt{n_2}} \b 1_{n_2} & \cdots & \b 0_{n_2}\\
\vdots & \vdots & \ddots & \vdots \\
\b 0_{n_h}& \b 0_{n_h} & \cdots & \frac{1}{\sqrt{n_h}} \b 1_{n_h}
\end{bmatrix} = \b O,
$$
which leads to a conclusion.

\t{Conclusion 3.} The constraint $\b F\trans \b U_\perp = \b O$ requires soft labels in each connected subgraph balanced and having an average of zero. Therefore, it will cause a mess to directly apply the Green-function method on a non-fully connected graph.

We try to turn a non-fully connected graph into a fully connected one with perturbation to solve the problem. Assume that any two pairs would gain a little bit of similarity because of the perturbation, that is,
$$
\b S^* = \b S + \mu \b 1_n {\b 1_n} \trans, 0 < \mu \ll 1.
$$

Adjusted with $\b S^*$, a renewed Laplacian matrix $\b L^*$ becomes:
\eqta\label{eq:L*}
\b L^* = \b L + n \mu \b I_n -\mu \b 1_n {\b 1_n}\trans,
\eqtb
where $\b I_n$ is the identity matrix.

It is clear that $\b L^*$ has the eigenvalues as
\eqta
\sigma^*_i = \left\{
\alna
0 &, &&i = 1 \\
n \mu&, &&2 \leq i \leq h \\
\sigma_i + n \mu&, &&h < i \leq n
\alnb
\right.,
\eqtb
and the eigenvectors the same as $\b L$. The dimension of null space is 1 and the problem turns into:
\eqta\label{eq:min*}
\min_{\b F\trans \b 1_n = \b 0} \r{Tr}(\b F\trans \b L^* \b F) - 2 \gamma \r{Tr}(\b F\trans \b Y).
\eqtb

\section{Accelerating Techniques on Large Graphs} \label{sec:acclr}

\noindent GSSL methods have high computational costs to deal with the similarity matrix and the Laplacian matrix when the number of samples is large. As is shown in Fig. (\ref{fig:Contri}) , we propose two techniques to accelerate the improved Green-function method. The method with the Gauss Elimination technique has the same result proved by theory and the time complexity is still $O(n^3)$ yet with a smaller constant coefficient. The method with the second one follows the anchor-based models\cite{nie2010general, Liu2010LargeGC, liu2016large, He2020FastSL} and gets a speed boost. The method accelerated by Anchored Graphs has a complexity of $O(nm^2)$ and needs to generate $m$ anchor points.

\subsection{Accelerating by Gauss Elimination}
\noindent The Green-function method needs to calculate the Moore-Penrose inverse for $\b L$, as is shown in Eq. (6). The calculation of the Moore-Penrose inverse usually needs the singular value decomposition(SVD) of $\b L$, and thus costs a lot of time.

By transforming the constraint $\b F\trans \b 1_n = \b 0$ into a term with the infinite coefficient $\eta$, problem (\ref{eq:min*}) becomes:
\eqta\label{eq:min_eta}
\min_{\b F} \r{Tr} (\b F \trans \b L^* \b F) - 2 \r{Tr} (\b F\trans \b Y) + \eta \r{Tr} (\b F \trans \b 1 {\b 1} \trans \b F),
\eqtb
where $\eta \rightarrow +\infty$. The minimum is achieved when:
\eqta\label{eq:F_eta}
\alna
& (\b L^* + \eta \b 1 \b 1\trans) \b F = \b Y \\
\Leftrightarrow &\b F = (\b L + n\mu \b I_n + \eta \b 1 \b 1\trans)^{-1} \b Y.
\alnb\eqtb

In this design, the soft label matrix $\b F$ can be obtained by solving Eq. (\ref{eq:F_eta}). The transformation of achieving $\b F$ from SVD to solving a linear equation enables the use of Gauss elimination.

Gauss elimination, also known as row reduction, is an algorithm for solving systems of linear equations. It consists of a sequence of operations performed on the corresponding matrix of coefficients.

The parameter $\eta \rightarrow + \infty$ makes the equation hard to tackle, so an $\eta \gg 1$ is used instead and leads to an approximate solution which is shown in \t{Algorithm 1}. With Gauss elimination, the time spent is reduced to one-tenth or even less compared with GF.

\begin{algorithm}
\renewcommand{\algorithmicrequire}{\textbf{Input:}}
\renewcommand{\algorithmicensure}{\textbf{Output:}}
\caption{Improved Green-function Method With Gaussian Elimination}
\begin{algorithmic}
\REQUIRE Sample set $\{\b x_1, \b x_2,...,\b x_n\} \subset \mathbb R^d$, a few labeled samples, and their labels $\{ y_1, y_2, ..., y_l \} \subset \mathbb N$.
\STATE Construct the similarity matrix $\b S$ with Guass kernel.
\STATE Choose a very small $\mu$. For example, let $\mu = 1\r{e}-5$. 
\STATE Choose a very large $\eta$. For example, let $\eta = 1\r{e}6$. 
\STATE Solve the linear equation $(\b L + n\mu \b I_n + \eta \b 1 \b 1\trans)\b F = \b Y$.
\STATE Denote $\b F = (\b f_1, \b f_2, ..., \b f_n)\trans$.
\FOR{$i = 1,2,...,n$}
\STATE $\hat y_i = \r{arg} \max\limits_j (\b f_i)_j$
\ENDFOR
\ENSURE Estimated classification $\{\hat y_i \}_{i=1}^n $
\end{algorithmic}
\end{algorithm}

\subsection{Accelerating by Anchored Graph}


If plugged in with Eq.(\ref{eq:L*}) and the definition of $\b L$, Eq. (\ref{eq:F_eta}) turns into:
\eqta\label{eq:F_eta_mu}
\b F = (\b D - \b S + n\mu \b I - \eta \b 1 \b 1\trans)^{-1} \b Y,
\eqtb
where $\eta \rightarrow +\infty$ and $0 \leqslant \mu \ll 1$ are two parameters. Note that the term $- \mu \b 1 \b 1\trans$ is omitted due to $\eta \rightarrow +\infty$.

Considering the format of Eq. (\ref{eq:F_eta_mu}), we would like to construct $m$ anchor points to fasten the algorithm where $m \ll n$. Anchor points are also located in the sample space $\mathcal X$, so the similarity between a sample and an anchor point could be calculated.

Denote $\b Z \in [0, 1]^{n \times m}$ as the similarity matrix measuring the underlying relationship between samples and anchor points, the element $Z_{il}$ in the $i$-th row, $l$-th column of which describes the metric similarity between the $i$-th sample and the $l$-th anchor. Following the design in \cite{Liu2010LargeGC}, we have
\eqta\label{eq:S_Z}
\b S = \b Z \b \Lambda^{-1} \b Z\trans, \b \Lambda = \r{Diag}(\b Z\trans \b 1_n).
\eqtb


We apply the Balanced $k$-means based Hierarchical $k$-means (BKHK) algorithm \cite{Nie2017UnsupervisedLG} to generate the $m$ anchor points in the sample space.

BKHK algorithm is a hierarchical clustering approach to recursively divide one cluster into two balanced clusters which contain the same number of samples.
The sample set is evenly divided into two clusters by the balanced $k$-means algorithm, then each of the two clusters is evenly divided into two smaller clusters.


BKHK works efficiently to find a given number of clusters, of which the centers can be seen as anchor points to represent the whole samples in the space.

With the set of anchor points $\mathcal P = \{\b p_1, \b p_2, ..., \b p_m\}$, we’ll construct the matrix $\b Z$ .

In the Nadaraya-Watson kernel regression\cite{hastie2009elements}, the matrix $\b Z$ can be defined as:
$$
Z_{il} = \frac{K(\b x_i, \b p_i)}{\sum_{l’ \in \langle i \rangle} K(\b x_i, \b p_{l’})}, \forall l \in \langle i \rangle,
$$
where $K(\cdot,\cdot)$ denotes an artificial kernel function and $\langle i \rangle \subset [1:m]$ denotes the set of neighbors of the $i$-th sample.

Usually, kernel $K$ is designed as a Gaussian kernel with heat parameter $\sigma$ so that $K(\b x_i, \b p_l) = \r{exp}(-||\b x_i - \b p_i||_2^2/2\sigma^2)$ is in $(0, 1]$. The heating parameter influences the effect of the algorithm based on the graph a lot\cite{Nie2009SemisupervisedOD}, so Nie \textit{et al.}\cite{Nie2016TheCL} proposed a parameter-free strategy by solving the problem:
\eqta\label{eq:min_z}
\min_{\b z_i\trans \b 1 = 1, Z_{il}>0}\sum_{l=1}^m (||\b x_i - \b p_l||_2^2 Z_{il} + \gamma_i {Z_{il}}^2),
\eqtb
where $\b z_i\trans$ denotes the $i$-th row of $\b Z$ and $\gamma_i$ is a parameter that need to be maximized. 


Following the detailed derivation in \cite{Nie2016TheCL}, problem (\ref{eq:min_z}) can be solved and $Z_{il}$ can be obtained as:
\eqta\label{eq:Z}
Z_{il} = \left\{ \alna &\frac{\mathcal E_{i,(k+1)}-e_{il}}{k\mathcal E_{i,(k+1)}-\sum_{j=1}^k \mathcal E_{i,(j)}},&& e_{il} < \mathcal E_{i,(k+1)} & \\ &0,&& \r{otherwise}& \alnb \right.,
\eqtb
when notating the distance as $e_{il} = ||\b x_i - \b p_l||_2^2$, the distance set as $\mathcal E_i = \{e_{i1}, e_{i2}, ..., e_{im}\}$ and the $j$-th smallest element in the set $\mathcal E_i$ as $\mathcal E_{i,(j)}$

It is of great benefit to construct $\b S$ with Eq.(\ref{eq:S_Z}) and Eq.(\ref{eq:Z}). Firstly, Eq. (\ref{eq:Z}) means fewer hyper-parameters, more efficient computation, and a scale-invariant result. Secondly, Eq. (\ref{eq:S_Z}) means the degree matrix $\b D$ equals to $\b I_n$ so that the Laplacian matrix $\b L$ equals the symmetric normalized Laplacian matrix $\b L_{sn} = \b D^{-1/2} \b L \b D^{-1/2}$ that not only works well for regular but also irregular graphs\cite{spielman2012spectral}. At last, the graph also satisfies the principles proposed in \cite{Liu2010LargeGC}.

For the convenience of following derivation, we denote Anchored Graphs matrix $\b B$ as $\b B = \b Z \b \Lambda^{-1/2}$ and rewrite Eq. (\ref{eq:S_Z}) into
$$
\b S = \b B \b B\trans,
$$
where $\b \Lambda^{-1/2}$ means the result of writing each element in the diagonal matrix $\b \Lambda$ into its reciprocal square root.

After obtaining the Anchored Graph $\b B$, Eq. (\ref{eq:F_eta_mu}) thus turns into
\eqta\label{eq:F_inv_n}
\alna
\b F &= (\b D - \b B \b B\trans + n\mu \b I + \eta \b 1 \b 1\trans)^{-1} \b Y \\
&= ((\b D + n\mu \b I) + \begin{bmatrix} \b B & \b 1_n \end{bmatrix} \begin{bmatrix} -\b I_m & \b 0 \\ \b 0\trans & \eta \end{bmatrix} \begin{bmatrix} \b B\trans \\ \b 1_n\trans\end{bmatrix})^{-1} \b Y,
\alnb \eqtb
which can be applied with Penrose’s matrix inversion lemma\cite{Penrose1955AGI}.

Penrose’s matrix inversion lemma allows us to efficiently find an approximation of the inverse of the matrix $\b A + \b B$ where the matrix $\b B$ can be approximated by a low-rank matrix $\b U \b C \b V$. When the matrixes $\b A$, $\b C$ and $\b A + \b U \b C \b V$ are all invertible, it proves that
$$
(\b A + \b U \b C \b V)^{-1} = \b A^{-1} - \b A^{-1} \b U (\b C^{-1} + \b V \b A^{-1} \b U)^{-1} \b V \b A^{-1}.
$$

For the sake of convenience, denote the diagonal matrix $\b D_\r{ir}$ as $\b D + n\mu \b I$. It is obviously that Eq. (\ref{eq:F_inv_n}) satisfies the requirements and can be transformed into:
\eqta\label{eq:F_inv_m}
\b F = (\b D_\r{ir}^{-1} - \b D_\r{ir}^{-1} \begin{bmatrix} \b B & \b 1_n \end{bmatrix} \b M^{-1} \begin{bmatrix} \b B\trans \\ \b 1_n\trans\end{bmatrix} \b D_\r{ir}^{-1} ) \b Y,
\eqtb
where $\b M = \begin{bmatrix} \b B\trans \b D_\r{ir}^{-1} \b B - \b I_m & \b B\trans \b D_\r{ir}^{-1} \b 1 \\ \b 1\trans \b D_\r{ir}^{-1} \b B & \b 1\trans \b D_\r{ir}^{-1} \b 1 + 1/\eta \end{bmatrix}$ and $\b D_\r{ir} = \b D + n\mu \b I_n$.

Noticing that $\eta$ tends to infinity and $\b D$ equals to $\b I_n$ in the graph, there’s $\b D_\r{ir} = (1+n\mu) \b I_n$, so Eq. (\ref{eq:F_inv_m}) becomes:
$$
\b F = \frac{1}{\theta} (\b Y - \begin{bmatrix} \b B & \b 1_n \end{bmatrix} \begin{bmatrix} \b B\trans \b B -\theta \b I_m & \b B\trans \b 1_n \\ \b 1_n\trans \b B & n \end{bmatrix}^{-1} \begin{bmatrix} \b B\trans \\ \b 1_n\trans\end{bmatrix} \b Y),
$$
where $\theta = 1 + n\mu$. The coefficient $\frac{1}{\theta}$ won’t affect the classification results, so we directly compute $\b F$ as:
\eqta\label{eq:F_result}
\b F = \b Y - \begin{bmatrix} \b B & \b 1_n \end{bmatrix} \begin{bmatrix} \b B\trans \b B -\theta \b I_m & \b B\trans \b 1_n \\ \b 1_n\trans \b B & n \end{bmatrix}^{-1} \begin{bmatrix} \b B\trans \\ \b 1_n\trans\end{bmatrix} \b Y.
\eqtb

The time complexity of the fastened algorithm shown in \t{Algorithm 2} is reduced to $\r{O}(nd\log m+nm^2)$, where $n$ denotes the number of all samples, $m$ denotes the number of anchor points when the number of classes is much smaller than the number of anchor points.

\begin{algorithm}
\renewcommand{\algorithmicrequire}{\textbf{Input:}}
\renewcommand{\algorithmicensure}{\textbf{Output:}}
\caption{Anchored Improved Green-function Method}
\begin{algorithmic}
\REQUIRE \REQUIRE Sample set $\{\b x_1, \b x_2,...,\b x_n\} \subset \mathbb R^{d} $, a few labeled samples, their labels $\{ y_1, y_2, ..., y_l \} \subset \mathbb N$, and the number of anchor points $m$.
\STATE Generate $m$ anchor points with BKHK and denote them as $\{\b p_1, \b p_2, ..., \b p_m\}$.
\FOR {$i = {1,2,...,n}$}
\FOR{$l = {1,2,...,m}$}
\STATE $e_{il} = ||\b x_i - \b p_l||_2^2$
\ENDFOR
\STATE Denote the set $\mathcal E_i$ as $\{e_{i1}, e_{i2}, ..., e_{im}\}$.
\STATE Denote the $j$-th smallest element in $\mathcal E_i$ as $\mathcal E_{i,(j)}$.
\STATE Calculate $Z_{il}$ as shown in Eq, (\ref{eq:Z}).
\ENDFOR
\STATE Calculate $\b B = \b Z \b \Lambda ^{-1/2}$ where $\b \Lambda = \r{Diag} (\b Z\trans \b 1)$.
\STATE Calculate $\b \theta = 1 + n\mu$.
\STATE Use either Eq.(\ref{eq:y1}) or Eq. (\ref{eq:y2}) to obtain $\b Y$.
\STATE Compute $\b F = (\b f_1, \b f_2, ..., \b f_n)\trans$ as shown in Eq. (\ref{eq:F_result})
\FOR{$i = 1,2,...,n$}
\STATE $\hat y_i = \r{arg} \max\limits_j (\b f_i)_j$
\ENDFOR
\ENSURE Estimated classification $\{\hat y_i \}_{i=1}^n $
\end{algorithmic}
\end{algorithm}

The whole algorithm has three parts:

1) To generate the $m$ anchor points from the features of samples with BKHK needs a time complexity of $O(nd \log m)$ and a space complexity of $O(nm)$.

2) To construct Anchored Graphs $\b B$ with anchor points and samples needs a time complexity of $O(nmd)$ and a space complexity of $O(nm)$.

3) To compute $\b F$ and classify all the samples needs a time complexity of $O(nm^2)$ and a space complexity of $O(nm)$.

To sum the above up, the algorithm has a total time complexity of $\r{O} (nd \log m + nm(m+d))$ and a total space complexity of $O(nm)$, making it possible to apply the algorithm on a large graph.

\section{Relationship With LLGC} \label{sec:Compare}
\noindent Our approach looks a little similar to but actually different from LLGC\cite{Zhou2003LearningWL}. As is shown in Fig. \ref{fig:Contri}, we will talk about their relationship in two aspects of the physical meaning and the outcome, which respectively lead to Conclusion 4 and 5.

In the aspect of the physical meaning, we use the \textit{label margin rule} instead of the \textit{fitting constraint} in LLGC. The \textit{fitting constraint} would require the soft labels $f$ of those labeled positive samples whose label $y$ is $+1$ to be close to 1, which may decrease the corresponding soft label $f$ when $f$ is large. Likewise, for the labeled negative samples, the margin between positive and negative samples cannot be maximized. However, our proposed rule always tries to push positive and negative samples away from each other and maximizes the margin between them.

For unlabeled samples, the label’s value $y$ equals 0. It is strange to require soft label $f$ of unlabeled samples close to the label’s value $y$ which equals $0$. Some algorithms \cite{He2021FastSL, Wang2022SemisupervisedLV} introduce regularization parameters to deal with the condition. With regularization parameters $\beta$, problem (\ref{eq:LLGC}) turns into:
$$
\min_{\b F} \sum_{i,j=1}^n S_{ij} || \b f_i - \b f_j ||_2^2 + \sum_{i=1}^n \beta_i ||\b f_i - \b y_i ||_2^2,
$$
where $\beta_i$ is the regularization parameter associated with the $i$-th sample.

\t{Conclusion 4.} Unlike LLGC and related approaches, \textit{label margin rule} doesn’t need to introduce additional parameters or set some of them to 0. \textit{label margin rule} automatically removes the effect of unlabeled samples on the loss.

In the aspect of the outcome, it can be proved that our approach is equivalent to LLGC in some cases. Looking at the solution Eq. (\ref{eq:F_LLGC}) of LLGC and the solution Eq. (\ref{eq:F_eta_mu}) of our approach, it could be found that they all contain a component of $\b L$ added by several times the identity matrix. In the other word, they can be expressed as:
\eqta\label{eq:compare}
\alna
&\b F_\r{LLGC} = (\b L + \gamma \b I)^{-1} \b Y, \\
&\b F_\r{ours} = (\b L + n \mu \b I + \eta \b 1 \b 1\trans)^{-1} \b Y,
\alnb
\eqtb
where $\mu \ll 1$ and $\eta \rightarrow + \infty$.

Let’s use the definitions in Eq. (\ref{eq:L_svd}) again to analyze the eigenvalues of $\b L$. Noticing that $\b I = \sum_{i=1}^n \b u_i \b u_i\trans$, $\sigma_1 = 0$ and $\b u_1 = \frac{1}{\sqrt n} \b 1$, Eq. (\ref{eq:compare}) will become:
\eqta
\alna
\b F_\r{LLGC} &= (\sum_{i=2}^n {\sigma_i \b u_i \b u_i\trans} + \gamma \sum_{i=1}^n \b u_i \b u_i\trans)^{-1} \b Y \\
&= (\frac{1}{n\gamma} \b 1 \b 1\trans + \sum_{i=2}^n {\frac{1}{\sigma_i + \gamma} \b u_i \b u_i\trans}) \b Y,
\alnb
\eqtb
and
\eqta
\alna
\b F_\r{ours} &= (\sum_{i=2}^n {\sigma_i \b u_i \b u_i\trans} + n\mu \sum_{i=1}^n \b u_i \b u_i\trans + n\eta \b u_1 \b u_1\trans)^{-1} \b Y \\
&= (\frac{1}{n^2 (\mu + \eta)} \b 1 \b 1\trans + \sum_{i=2}^n {\frac{1}{\sigma_i + n\mu} \b u_i \b u_i\trans}) \b Y
\alnb
\eqtb

In many cases, we will choose the same number of labeled samples in each class, which means $\b 1\trans \b Y = (l/c) \b 1\trans,$
where $l/c$ is a positive integer that denotes the number of labeled samples for each class. So when $\gamma$ equals to $n\mu$, we get
$$
\b F_\r{LLGC} - \b F_\r{ours} = \frac{l}{c} (\frac{1}{n\gamma}-\frac{1}{n^2 (\mu + \eta)}) \b 1 \b 1\trans,
$$
which proves the equivalence between the two methods.

\t{Conclusion 5.} When $\gamma$ equals $n \mu$ and there is the same number of labeled samples in each class, our method is equivalent to LLGC. However, our method only needs to choose the $\mu$ much less than the edge weights yet greater than zero to guarantee stability, while LLGC needs to adjust the parameter $\gamma$.

\section{Experiments} \label{sec:experiment}

\noindent In this section, we conduct experiments to prove the efficiency of our approach and our conclusions as follows:
1) Our proposed method can achieve similar or even better results to GF and runs much faster, especially when there’re fewer labeled samples.
2) It’s irrational to apply the original Green-function method to non-fully connected graphs.
3) The label margins of our proposed method are bigger than those of LLGC\cite{Zhou2003LearningWL}.
4) Anchor technique can not only reduce the time and space complexity but usually make the method perform better as well.

To count the run time, all experiments are performed on a Windows 10 computer with a 3.60GHz Intel(R) Core(TM) i7-7700 CPU and 32.0 GB RAM, python 3.6.

\subsection{Experiments on Classification Accuracy and Time Cost} \label{ex:acc}

\begin{table}[b]
\centering
\caption{The Description of Six Datasets} \label{t:dataset}
\begin{center}
\begin{tabular}{ccccc}
\hline
\t{Dataset} & \t{Samples} & \t{ Classes} & \t{Dimensions} & \t{Anchors} \\ 
\hline
Balance & 625 & 3 & 4 & 21\\
MobileKSD & 2,856 & 56 & 71 & 2,856\\
USPS & 9,298 & 10 & 256 & 1,024 \\
CsMap & 10,846 & 6 & 29 & 1,024 \\
PhishingWeb & 11,055 & 2 & 30 & 50\\ 
Swarm & 24,016 & 4 & 2400 & 1024 \\ \hline
\end{tabular}
\end{center}
\end{table}

\begin{table*}[t]
\centering
\caption{Accuracy(\%) ± Standard Deviation(\%) of Different Approaches on Six Datasets\\ (“OM” Means “Out-of-memory Error”)} \label{t:acc}
\begin{center}
\begin{tabular}{cccccc}
\hline
\t{Dataset} & \t{LLGC} & \t{HF} & \t{GF} & \t{GF(G)} & \t{GF(A)} \\ \hline		
Balance&65.78±3.61&66.72±5.73&64.87±1.26&64.87±1.26&\t{69.71±2.39}\\
MobileKSD&34.04±0.83&35.23±0.89&31.45±0.71&31.45±0.71&\t{35.78±0.52}\\
USPS&91.14±0.71&89.14±1.07&\t{91.61±0.46}&\t{91.61±0.46}&89.69±1.02\\
CsMap&58.94±1.40&54.76±3.07&54.70±4.47&54.70±4.47&\t{59.63±1.09}\\
PhishingWeb&69.90±3.92&56.40±6.14&65.22±6.13&64.98±6.05&\t{72.71±9.84}\\
Swarm&93.76±2.68&92.70±2.18&OM&86.31±5.56&\t{96.58±0.89}\\\hline
\end{tabular}
\end{center}
\end{table*}

\begin{table}[t]
\centering
\caption{F1-macro of Different Approaches on Six Datasets (“OM” Means “Out-of-memory Error”)} \label{t:F1}
\begin{center}
\begin{tabular}{cccccc}
\hline
\t{Dataset} & \t{LLGC} & \t{HF} & \t{GF} & \t{GF(G)} & \t{GF(A)} \\ \hline	
Balance&0.6148&\t{0.6152}&0.5937&0.5937&0.5941\\
MobileKSD&0.3305&0.3506&0.3119&0.3119&\t{0.3513}\\
USPS&0.9052&0.8883&\t{0.9108}&\t{0.9108}&0.8878\\
CsMap&0.648&0.6297&0.6025&0.6025&\t{0.6545}\\
PhishingWeb&0.7004&0.6033&0.6581&0.6562&\t{0.7232}\\
Swarm&0.9205&0.9131&OM&0.8594&\t{0.9415}\\\hline
\end{tabular}
\end{center}
\end{table}

\begin{table}[t]
\centering
\caption{Time Cost(Seconds) of Different Approaches on Six Datasets (“OM” Means “Out-of-memory Error”)} \label{t:time}
\begin{tabular}{cccccc}
\hline
\t{Dataset} & \t{LLGC} & \t{HF} & \t{GF} & \t{GF(G)} & \t{GF(A)} \\ \hline
Balance&0.03&\t{0.02}&0.13&\t{0.02}&0.26\\
MobileKSD&0.86&0.59&7.37&\t{0.44}&0.62\\
USPS&16.86&16.75&140.86&7.18&\t{1.24}\\
CsMap&273.6&235.99&211.88&10.65&\t{0.87}\\
PhishingWeb&26.99&26.76&223.01&11.7&\t{0.83}\\
Swarm&311.14&283.44&OM&98.22&\t{20.48}\\\hline
\end{tabular}
\end{table}

\noindent We use 6 real-world datasets containing various numbers of samples from various domains to evaluate the approach. The number of samples, the number of classes, and the dimension of the sample space are listed in TABLE \ref{t:dataset}. Their sizes are no more than 30,000.

1) Balance Scale\cite{Dua:2019}. It is an abstract dataset for psycho logical experiments. Each example is to imagine a scale when weights and distances from the center of the two objects are known. The 4 attributes: the left weight, the left distance, the right weight, and the right distance are used to judge whether the scale is left-leaning. right-leaning or balanced. It’s a small sized dataset containing 625 samples. We’ll call it Balance for short.

2) MEU-Mobile KSD\cite{Dua:2019}. It contains keystroke dynamics data collected on a touch mobile device (Nexus 7). The sample consists of 71 features such as Hold (H), Pressure (P), Finger Area (A), and so on. It’s a small-sized dataset containing 2856 records, 51 records per class for 56 classes. We call it MobileKSD for short.

3) USPS\cite{Hull1994ADF}.. It is a digit dataset that contains 9,298 16×16 pixel grayscale samples of handwritten digits. It covers 10 classes and each class has 708-1,553 images. We use it as a medium-sized dataset.

4) Crowdsourced Mapping\cite{johnson2016integrating}. It contains satellite imagery and georeferenced polygons which need to be classified into 6 different kinds of land covers. In total, there are 10,846 samples but we calculate accuracty using only 300 samples from the test set because the others contains noise. We call it CsMap for short.

5) Phishing Websites\cite{Dua:2019}. The samples it contains are extracted from websites. 30 attributes have been summarized to determine whether it is a phishing website. It’s a medium-sized dataset containing 11,055 websites. We’ll call it PhishingWeb for short.

6) Swarm Behaviour\cite{Dua:2019}. It is a dataset about swarm behavior. 200 boids' features are summarized as 2400 attributes. Each situation needs to classify whether boids are aligned and whether boids are grouped. There're 24017 situations. We call it Swarm for short.

In addition to the original Green-function method and the two improved versions proposed in this paper, we also chose LLGC and HF as the experimental targets among graph-based semi-supervised learning methods. The parameters and details of them are shown as follows:

1) GF. The similarity matrix S in the original Green-function method is constructed by using the Gaussian kernel and k-nearest neighbor algorithm like
\eqta\label{eq:S_rbf}
S_{ij} = \exp (-\frac{1}{2\sigma^2} ||\b x_i - \b x_j||_2^2), j \in \langle i \rangle \; \r{or} \; i \in \langle j \rangle,
\eqtb
where $2 \sigma^2$ is chosen to be $\r{var}_{||\b x_i - \b x_j||_2^2 \neq 0}(||\b x_i - \b x_j||_2^2) / d$ to avoid overfit or underfitaverage of all the distances and $\langle i \rangle$ denotes the set of neighbors of the $i$-th sample of which the size is chosen to be $20$. 
It performs badly on non-fully connected graphs and is shown in Section \ref{ex:non-fully}, so we add the perturbation of $\mu$ to the whole graph after it is constructed. For 
Iterative methods to compute the pseudo-inverse always converge to trivial solutions when the size of $S$ is large. So we apply SVD to obtain the pseudo-inverse.
 
2) GF(G). The Green-function method using Gause elimination also takes the graph constructed by Eq. (\ref{eq:S_rbf}) and meets out-of-memory error on large-sized datasets.
The perturbation of $\mu = 1\r{e}-5$ is added while the parameter $\eta \gg 1$ is set to be $1\r{e}6$ which is large enough but won't cause data overflow within 32-bit float type. 

3) GF(A). The Green-function method on the Anchored Graphs uses the same number of neighbors which means $k = 20$. Anchored Graphs need to generate several anchor points. For datasets that have fewer classes, fewer anchor points bring high performance; and vice versa. The number of anchor points we generated is also shown in TABLE \ref{t:dataset}. More experiments about anchor points are shown later in Section \ref{ex:anchors}.

4) LLGC\cite{Zhou2003LearningWL}. Learning with Local and Global Consistency has been introduced in section \ref{llgc}. The parameter of $\gamma$ is set to be $1$. The graph it needs is constructed as that of GF.

5) HF\cite{Zhu2003SemiSupervisedLU}. Semi-Supervised Learning Using Gaussian Fields and Harmonic Functions is to acquire a thermodynamic equilibrium when treating labeled samples as heat sources. We call it HF for short. The graph it needs is constructed as that of GF.

We compare the approaches from the perspective of accuracy, stability, efficiency, precision, and recall.

\begin{figure*}[b]
\centering
\subfigure[Balance]{
\includegraphics[width=0.3\textwidth]{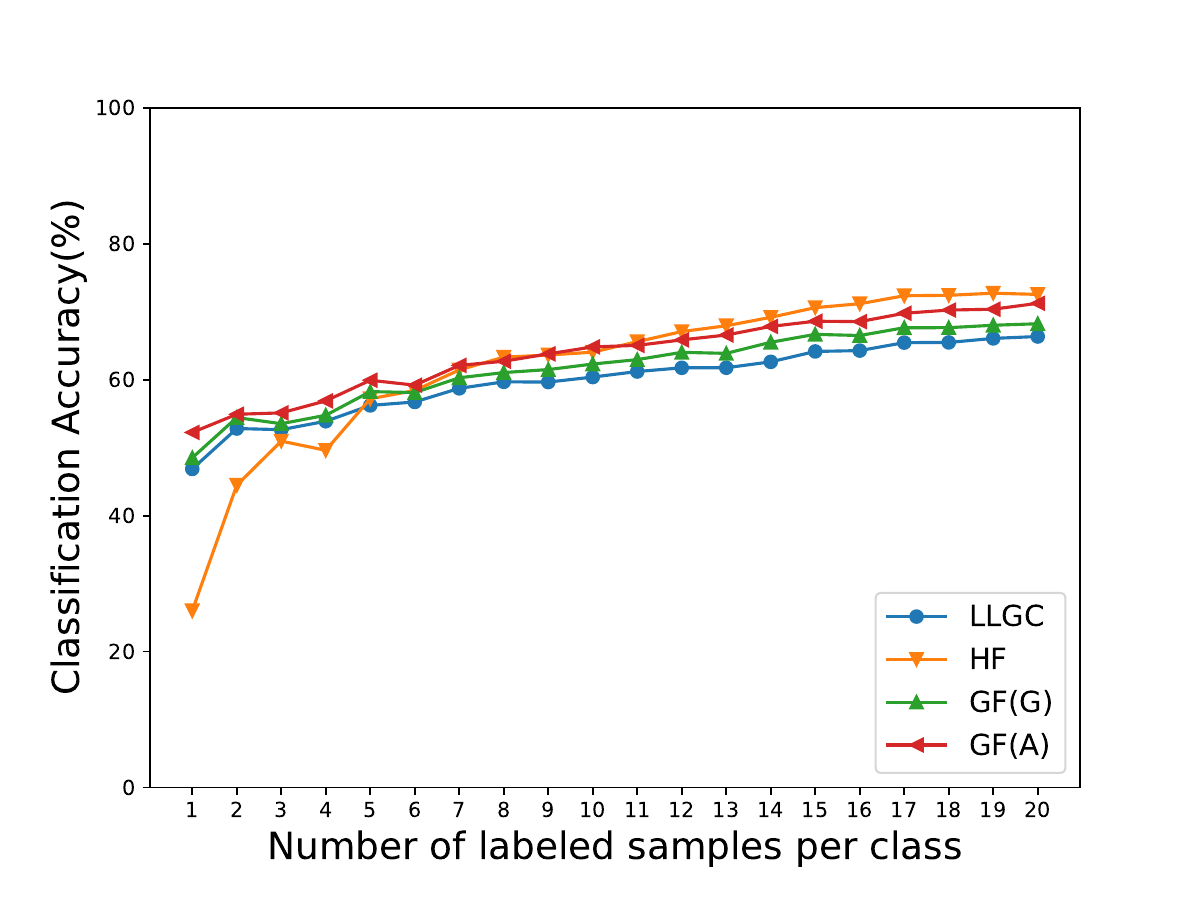}
}
\subfigure[MobileKSD]{
\includegraphics[width=0.3\textwidth]{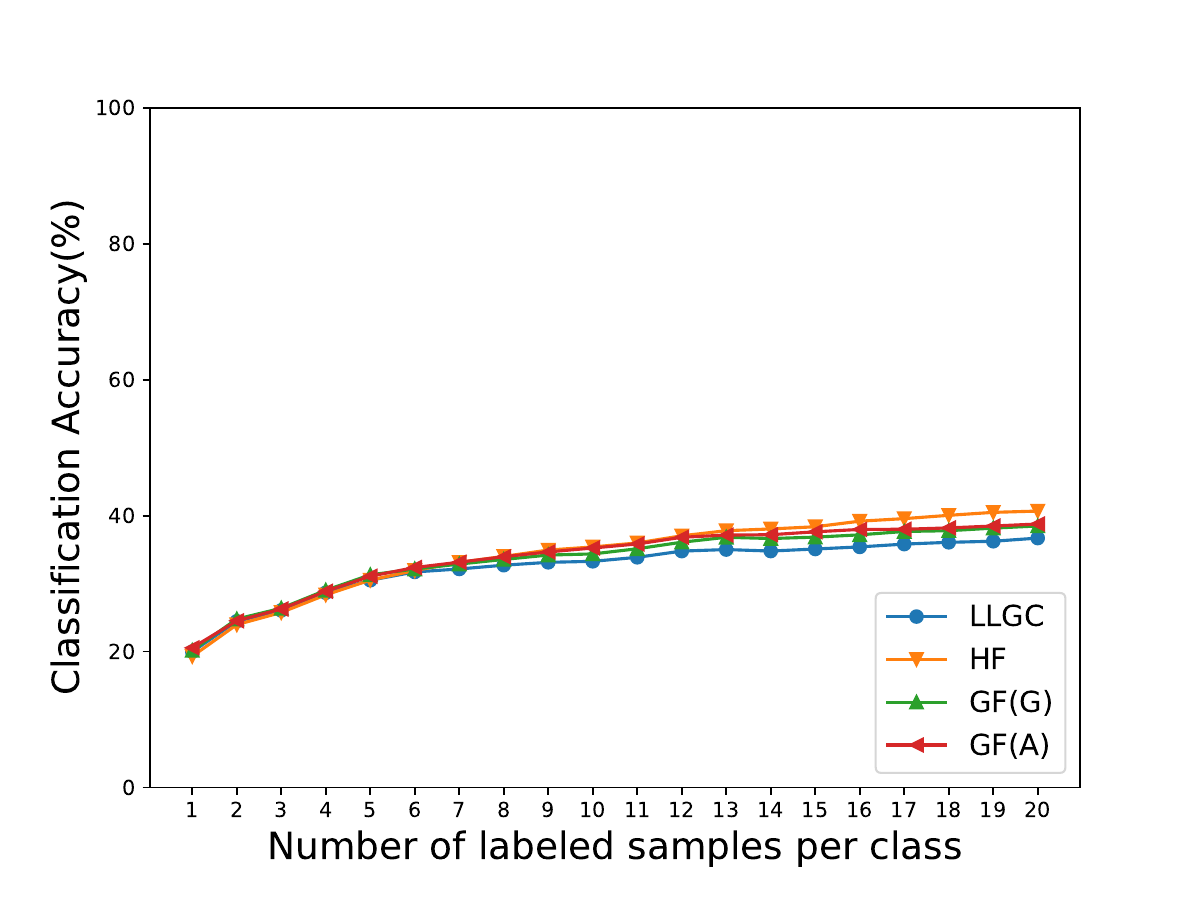}
}
\subfigure[USPS]{
\includegraphics[width=0.3\textwidth]{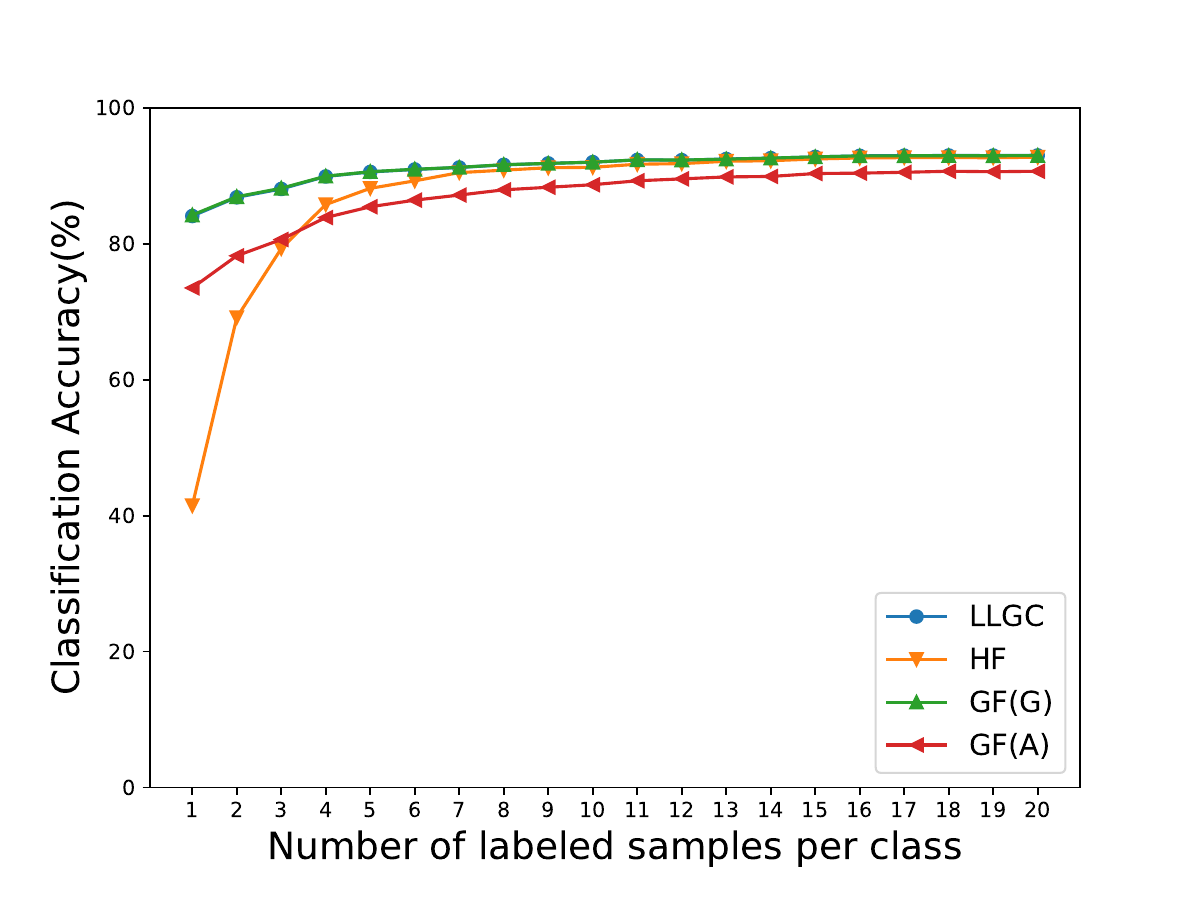}
}
\subfigure[CsMap]{
\includegraphics[width=0.3\textwidth]{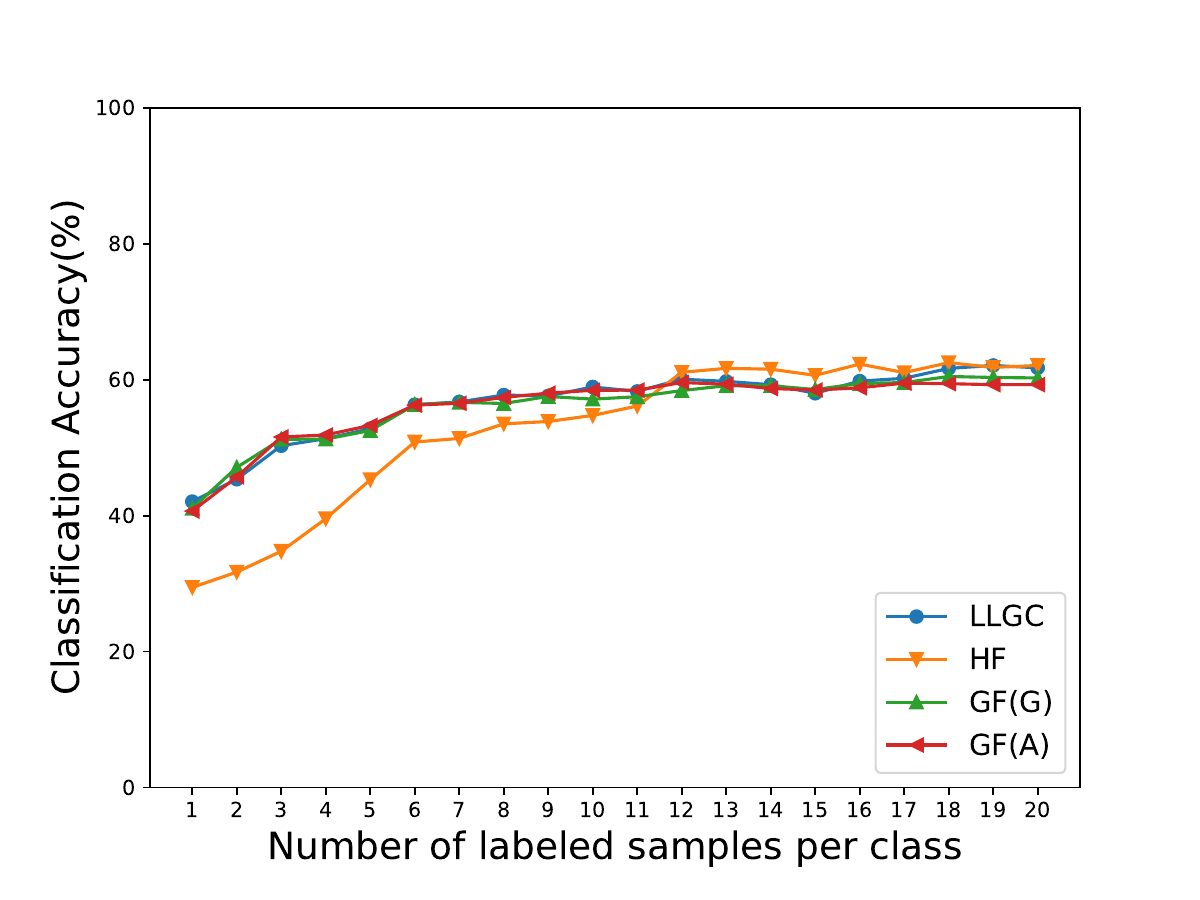}
}
\subfigure[PhishingWeb]{
\includegraphics[width=0.3\textwidth]{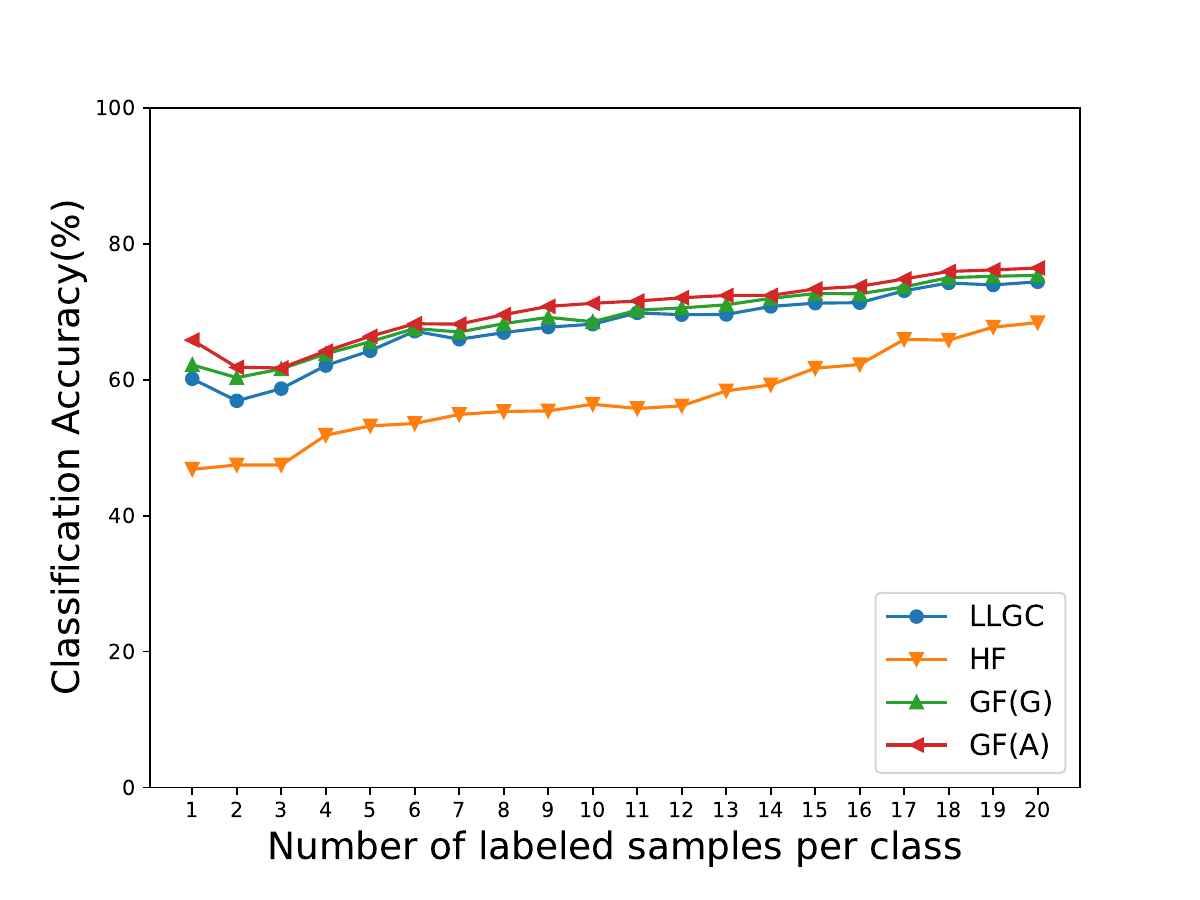}
}
\subfigure[Swarm]{
\includegraphics[width=0.3\textwidth]{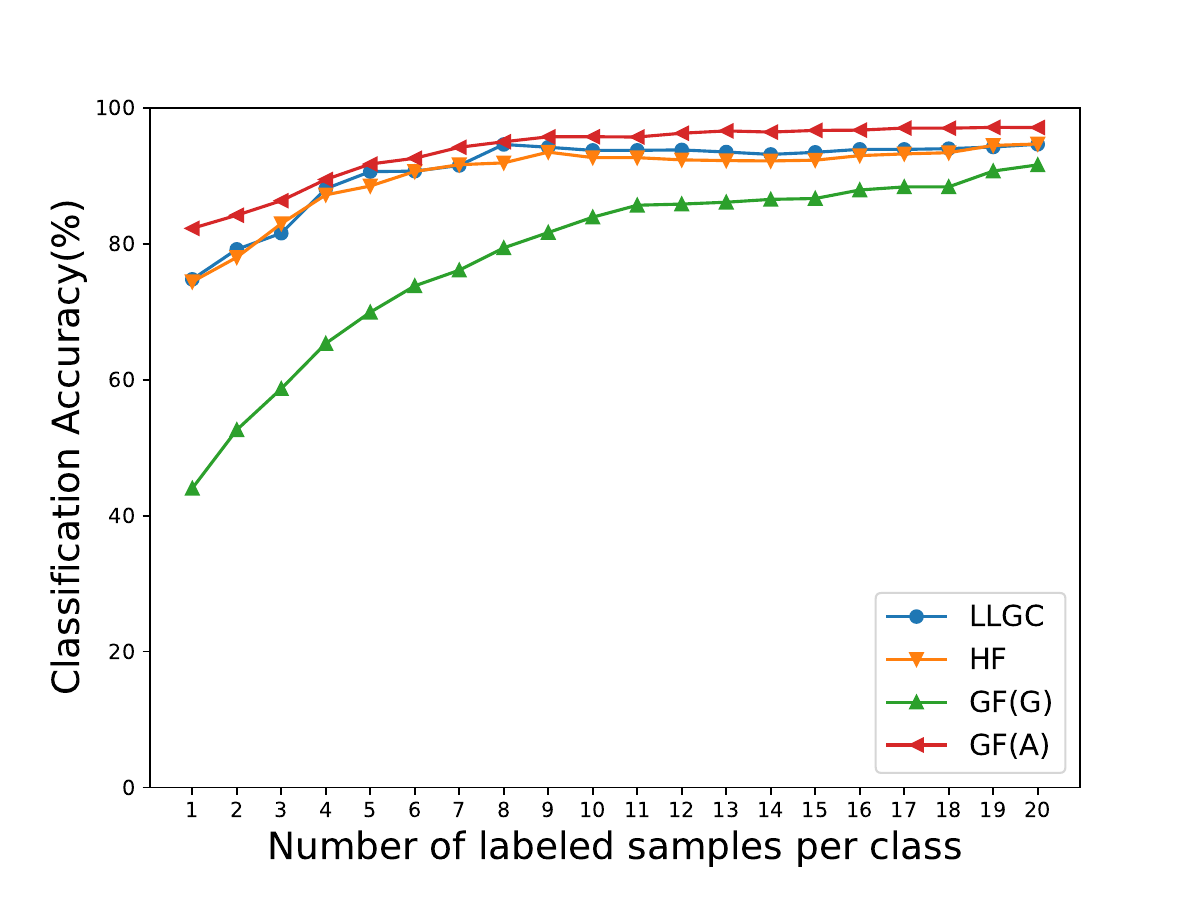}
}
\caption{Average classification accuracy versus the number of labeled samples per class.}
\label{fig:acc-label}
\end{figure*}

As is shown in TABLE \ref{t:acc}, classification accuracy is used to measure the performance of approaches and is the average of those of 5 trials, of which the standard deviation measures stability. In each trial, 10 samples for each class are randomly selected and labeled, that is, $l = 10c$. Only the unlabeled samples count. The highest accuracy is bolded. In the experiments, GF and GF(G) are validated to have almost the same result, which verifies \t{Conclusion 1}. There are subtle differences due to the numerical instability of the SVD
algorithm. Because of the different ways of constructing the graph, GF(A) has different but close results to the other two.

Following \cite{He2021FastSL, xu2019multi}, we use F1-macro to measure the precision and recall of these multi-classification problems. It shows up as
$$
\r{F1-macro} = \frac{2 \times \r{precision}_\r{macro} \times \r{recall}_\r{macro}}{\r{precision}_\r{macro} + \r{recall}_\r{macro}}.
$$
where $\r{precision}_\r{macro}$ and $\r{recall}_\r{macro}$ are respectively the average of precision and recalls of $c$ binary problems generated by the one-aganist-the rest strategy in a $c$-class classification.
The result is in TABLE \ref{t:F1} and the best ones are bolded.

In TABLE \ref{t:time}, efficiency is measured by the average time costs. The time costs cover the construction of graphs and the computing procedure. For all datasets, GF(G) is always faster than GF, proving the effectiveness of Gauss Elimination. The time cost of GF(A) does not have significant advantages for small-sized datasets because of the additional procedure of BKHK, while is quite faster than all the other semi-supervised algorithms for datasets containing more than 10,000 samples.

According to the three tables mentioned above, we could find that our proposed methods GF(G) and GF(A) achieve comparable performance and cost much less time.

In this part, we also conduct experiments to validate the equivalence between coding ways and achieve the same results, which verifies \t{Conclusion 2}.

When there are fewer labeled samples, our proposed method can still work well. For all datasets shown in Fig. \ref{fig:acc-label}, we set the number of labeled samples per class from 1 to 20 and evaluate them in terms of average classification accuracy. Each point in the figures represents the average accuracy of 5 trials with the labeled samples chosen randomly. According to common sense, accuracy increases as more samples are labeled.

\subsection{Experiments on Non-fully Connected Graphs}\label{ex:non-fully}

\begin{figure*}[t]
\centering
\subfigure[Ground Truth]{
\includegraphics[width=0.23\textwidth]{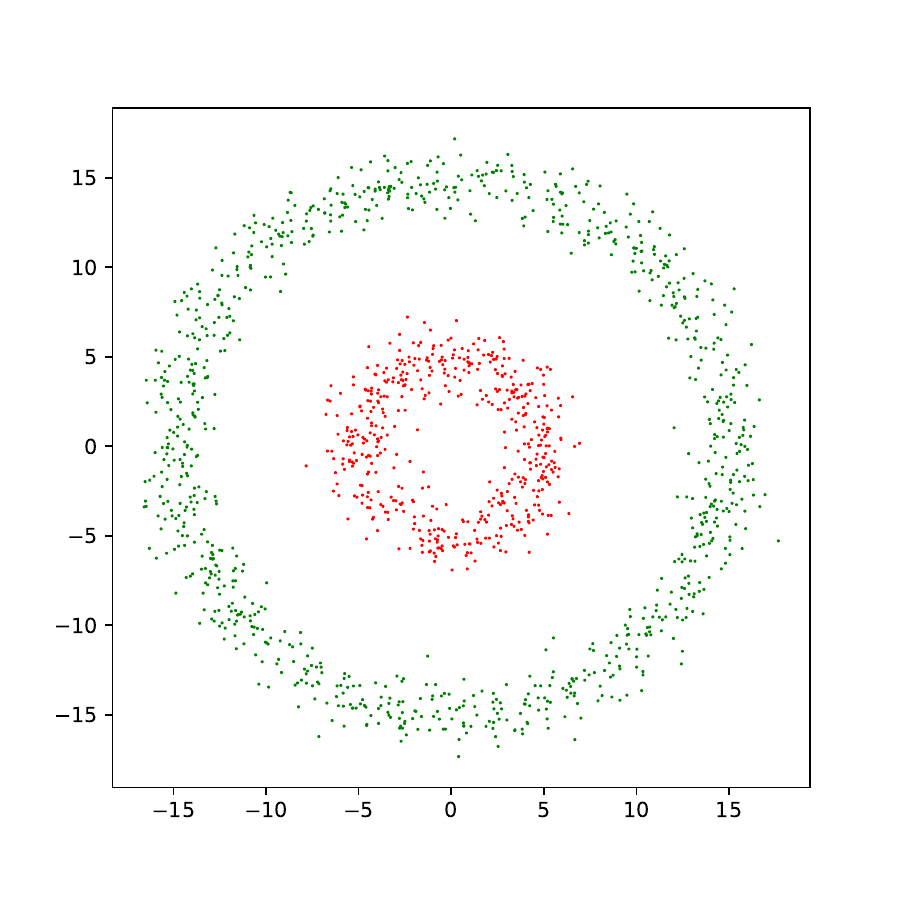}
}
\subfigure[Connections]{
\includegraphics[width=0.23\textwidth]{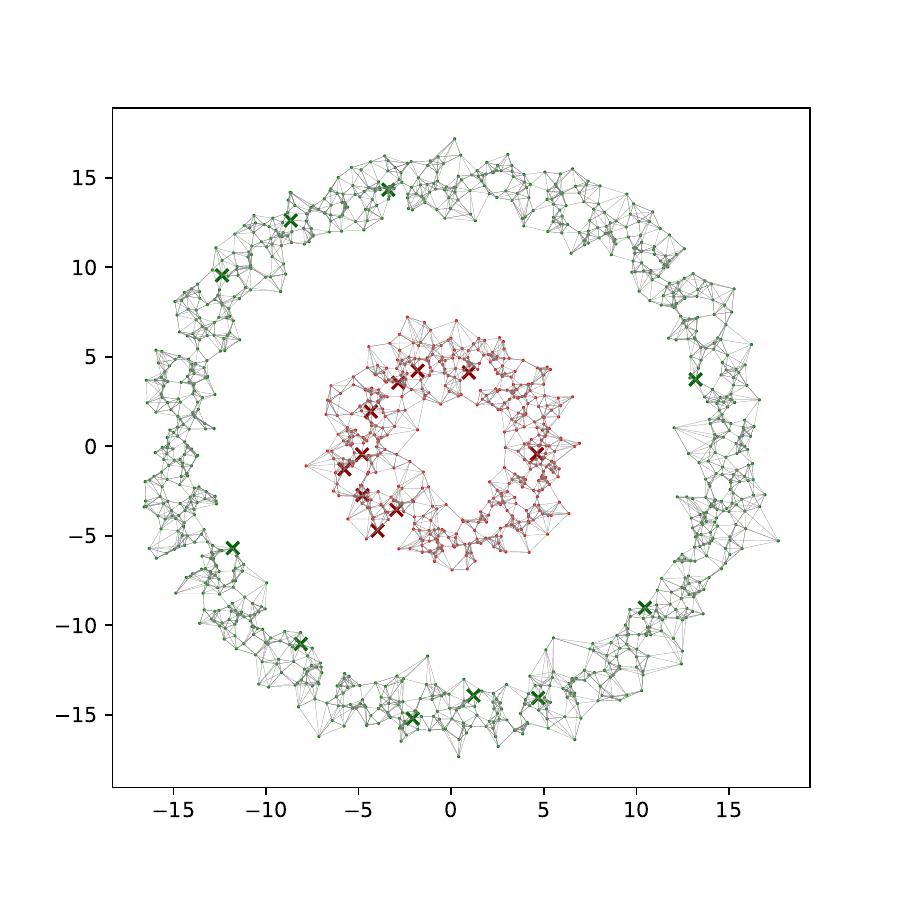}
}
\subfigure[Original GF]{
\includegraphics[width=0.23\textwidth]{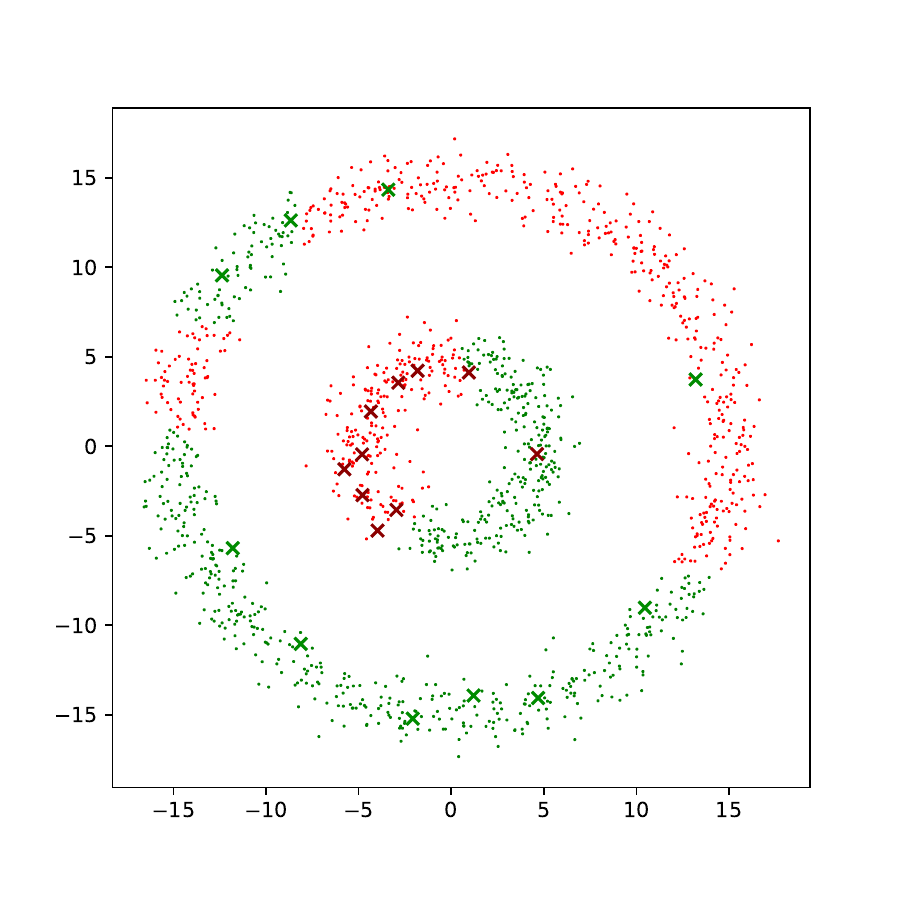}
}
\subfigure[Improved GF]{
\includegraphics[width=0.23\textwidth]{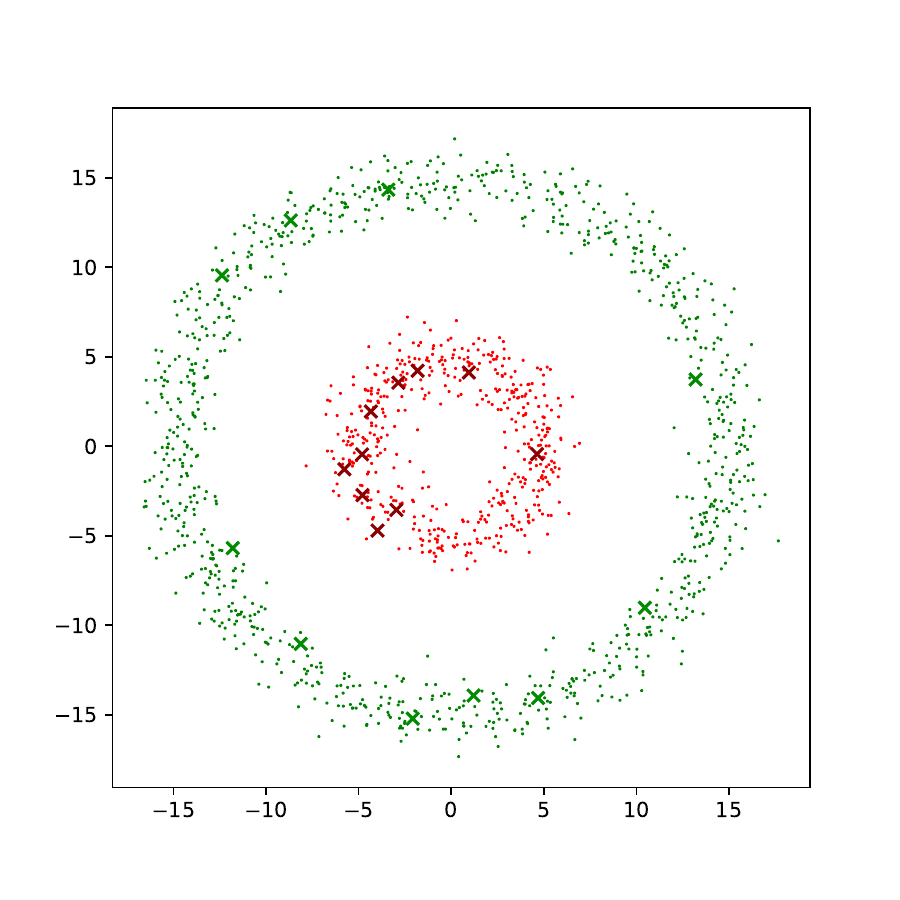}
}
\caption{Two-Ring Dataset Which is Non-fully Connected.}
\label{fig:Two-Ring}
\end{figure*}

\begin{figure*}[t]
\centering
\subfigure[Ground Truth]{
\includegraphics[width=0.23\textwidth]{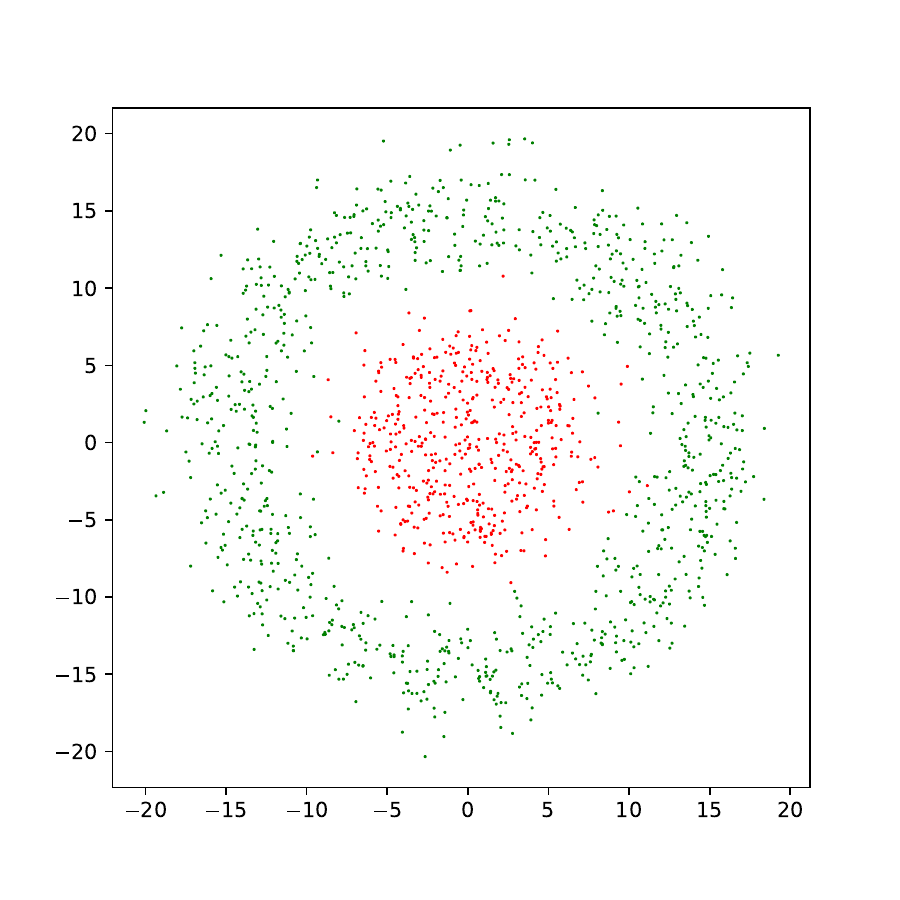}
}
\subfigure[Connections]{
\includegraphics[width=0.23\textwidth]{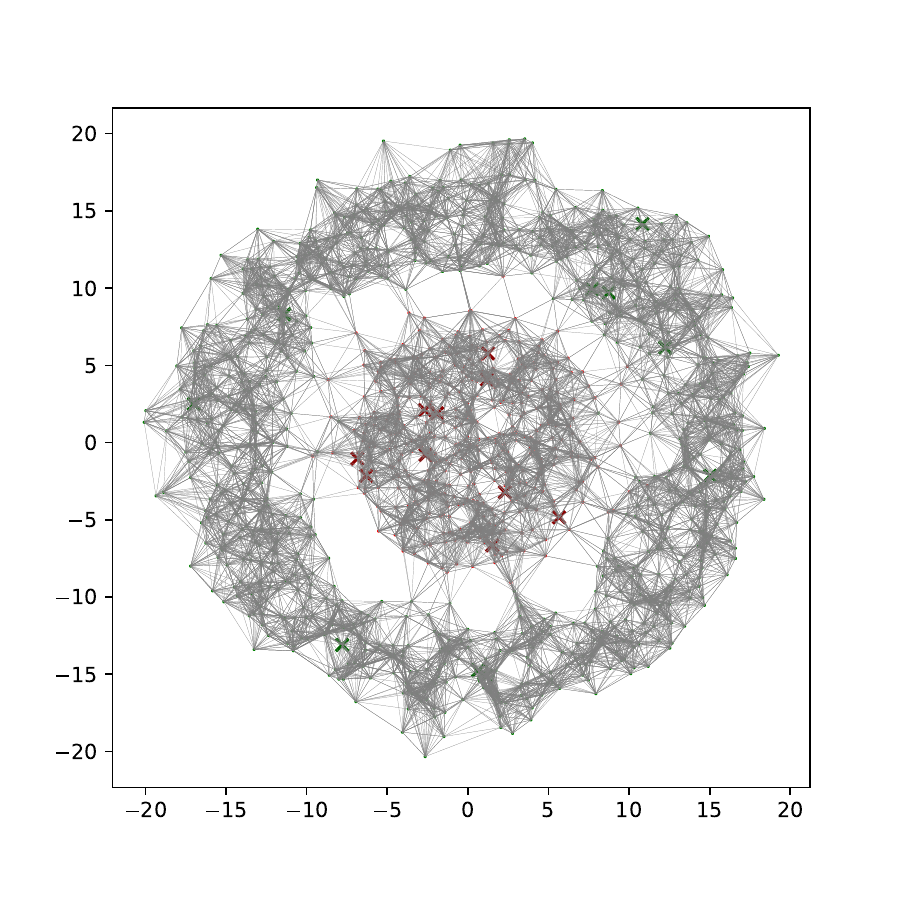}
}
\subfigure[Original GF]{
\includegraphics[width=0.23\textwidth]{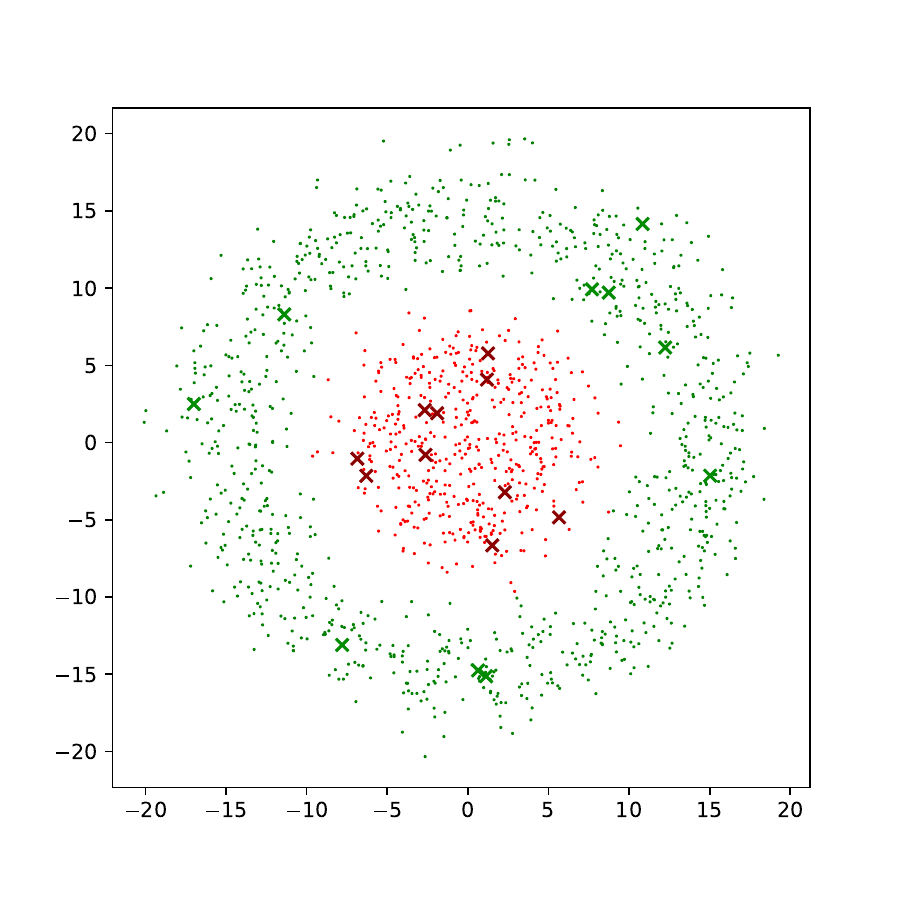}
}
\subfigure[Improved GF]{
\includegraphics[width=0.23\textwidth]{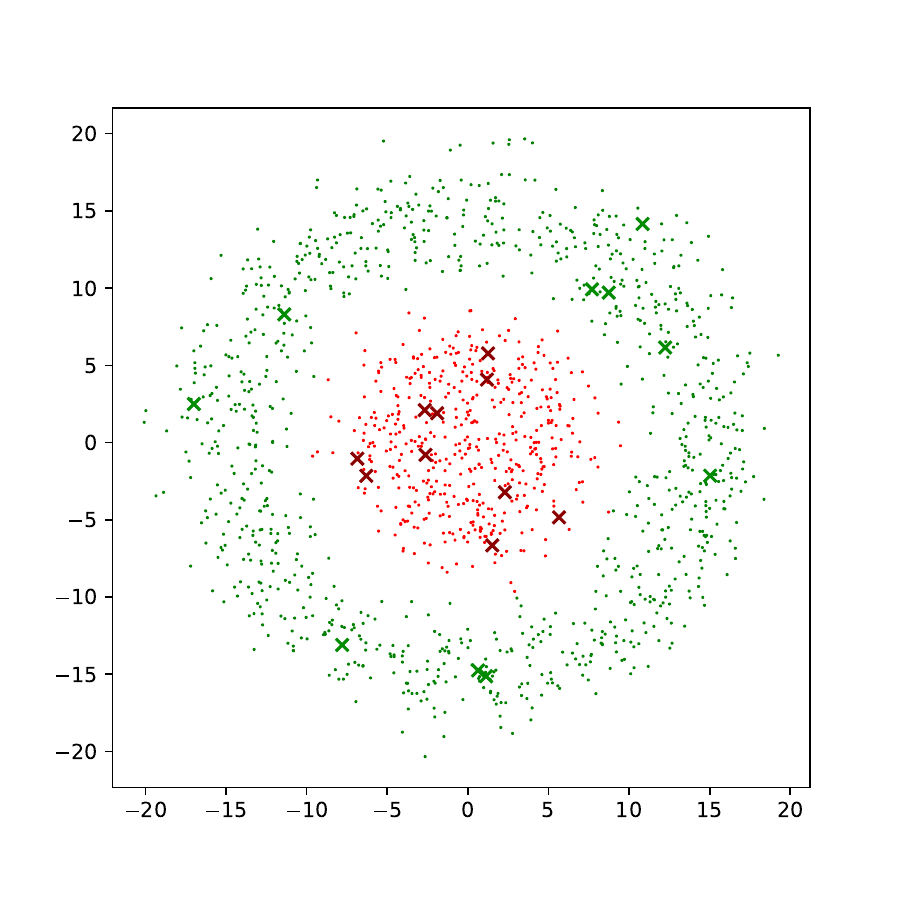}
}
\caption{Two-Ring Dataset Which is Fully Connected.}
\label{fig:Two-Ring-2}
\end{figure*}

\noindent We use two toy datasets to show the irrationality of applying the original Green-function method to a non-fully graph. They are shown as Fig. \ref{fig:Two-Ring} and Fig. \ref{fig:Two-Ring-2}. For both, the inner and outer rings respectively consist of 500 positive samples (the red points) and 1000 negative samples (the green points) but the density of rings differ. We set $k$ and $l$ as 20.

For each group of pictures, (a) describes the ground truth. In (b), a line is drawn between each pair of points whose similarity is not zero to show the connecting blocks. (c) is the result of the original GF and (d) is the result of the improved GF with perturbation $\mu=1\r{e}-5$. The forks stand for the randomly selected labeled samples.

In Fig. \ref{fig:Two-Ring}, the constraint $\b F\trans \b U_\perp = \b O$ of the original GF results in that each independent connecting block is divided into all classes, leading to terrible results. For a binary classification like the Two-Ring dataset, the sum of indicators is zero for each ring, so each of them has a positive part and a negative part. On the other hand, improved GF introduces a small perturbation to connect all points, and then performs well in the experiments.

For usual methods, it is harder to classify the dataset in Fig. \ref{fig:Two-Ring-2} than that in Fig. \ref{fig:Two-Ring} because of the more sparse rings. However, the two methods can both achieve good results in Fig. \ref{fig:Two-Ring-2} while the original method performs bad in Fig. \ref{fig:Two-Ring}, which verifies \t{Conclusion 3}. Besides, the situation like Fig. \ref{fig:Two-Ring} occurs frequently. Sparse graphs tend to exhibit high quality because of much less spurious connections\cite{Zhu2005SemiSupervisedLL}, so the number of nearest neighbors $k$ won't be very large. Therefore, it makes the graph easily become non-fully connected, especially for large-sized datasets with high dimensions and thus causes a mess.

\subsection{Experiments to Validate the Label Margins}

\noindent We use the 6 real-world datasets to verify Conclusion 4 which tells that the Label Margin constraint creates a larger gap between the positive and negative samples than the Consistency constraint. We propose a simple but efficient metric:
\eqta
\alna
&\r{LM}_l = \frac{1}{c} \sum_{j=1}^c (\frac{\sum_{i=1}^l I(y_i = c)F_{ij}}{\sum_{i=1}^l I(y_i = c)} - \frac{\sum_{i=1}^l I(y_i \neq c)F_{ij}}{\sum_{i=1}^l I(y_i \neq c)}), \\
&\r{LM}_u = \frac{1}{c} \sum_{j=1}^c (\frac{\sum_{i=l+1}^n I(\tilde y_i = c)F_{ij}}{\sum_{i=l+1}^n I(\tilde y_i = c)} - \frac{\sum_{i=l+1}^n I(\tilde y_i \neq c)F_{ij}}{\sum_{i=l+1}^n I(\tilde y_i \neq c)}),
\alnb
\eqtb
where $I(y_i = c)$ equals 1 if $y_i = c$ is true and 0 if $y_i = c$ is false, and $I(y_i \neq c)$ is the opposite. Unlabeled samples have unknown $y$, so we use $\tilde y$ to express their real class.

This metric is to calculate the difference between the respective indicators of the positive and negative samples. The larger the metric, the better the discrimination. As is shown in TABLE \ref{t:margin}, $\r{LM}_l$ is the Label Margin among labeled samples, and $\r{LM}_u$ is that among unlabeled samples. The parameters of LLGC and GF(G) are the same as those described in Section \ref{ex:acc} and the initial $\b Y$ is coded as $\b Y^{(1)}$ in Eq.(\ref{eq:y1}). From the table, we could find that our proposed method GF(G) always has larger Label Margins among labeled samples than LLGC. As a result, the Label Margins among unlabeled samples of GF(G) also prevail and are consistent with the result in TABLE \ref{t:acc} and \ref{t:F1}. The performance of GF(G) is consistent with that of the original GF, so we only draw the curve of one of them.
 	
\begin{table}[t]
\centering
\caption{Label Margin Amoung Labeled Smaples and Unlabeled Samples.} \label{t:margin}
\begin{center}
\begin{tabular}{ccccccc}
\hline
\multirow{2}{*}{\t{Dataset}}&\multicolumn{3}{c}{$\b{LM_l}$}&\multicolumn{3}{c}{$\b{LM_u}$}\\ \cline{2-7}
 &\t{LLGC}&~&\t{GF(G)}&\t{LLGC}&~&\t{GF(G)}\\ \hline
Balance&1.678&$<$&13.2357&0.0049&$<$&1.3686\\
MobileKSD&0.0693&$<$&0.0776&0.0129&$<$&0.0189\\
USPS&0.0651&$<$&0.0775&0.0077&$<$&0.0174\\
CsMap&0.0685&$<$&0.0845&0.0051&$<$&0.0118\\
PhishingWeb&0.4754&$<$&1.7166&0.001&$<$&0.0464\\
Swarm&0.0822&$<$&0.208&0.0045&$<$&0.0482\\\hline
\end{tabular}
\end{center}
\end{table}

\subsection{Experiments About Anchors and Computing Time} \label{ex:anchors}

\begin{figure*}[b]
\centering
\subfigure[Balance]{
\includegraphics[width=0.3\textwidth]{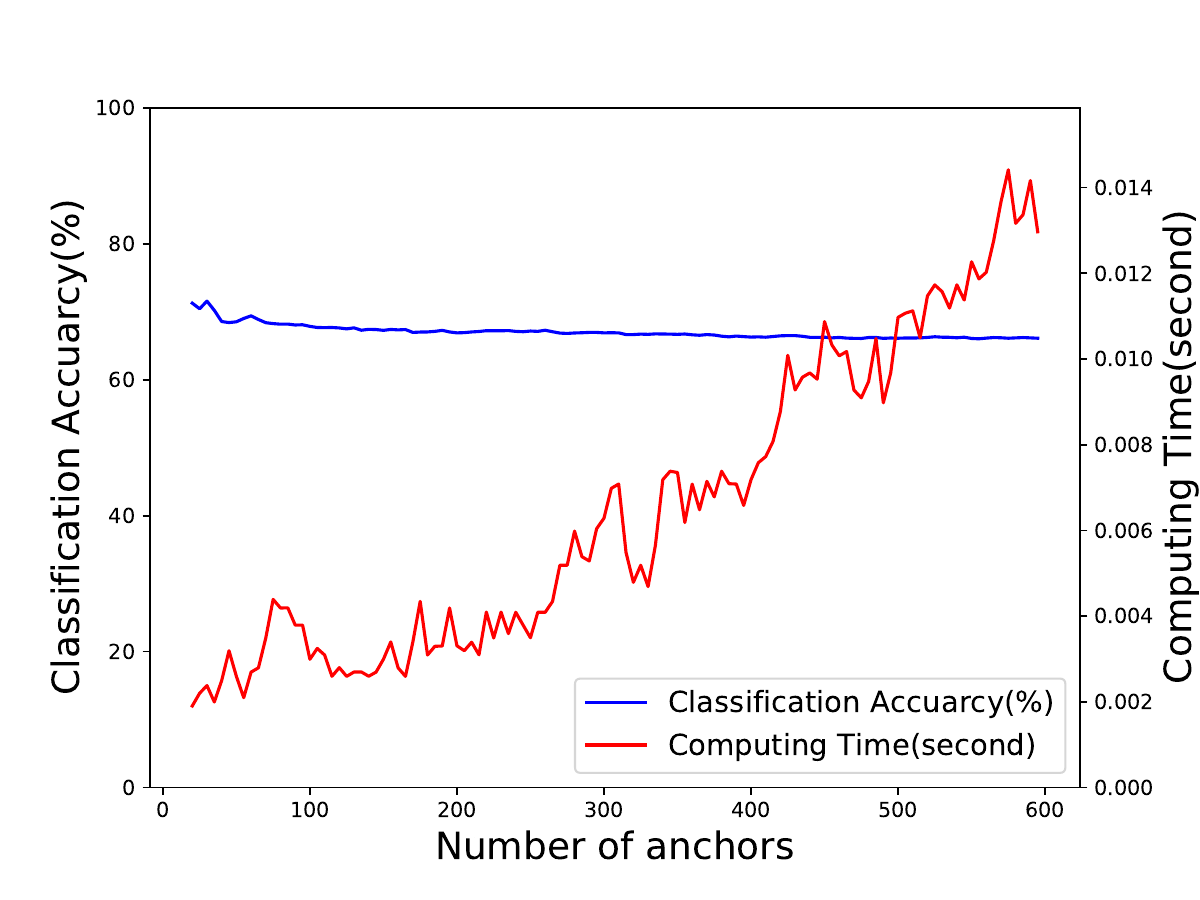}
}
\subfigure[MobileKSD]{
\includegraphics[width=0.3\textwidth]{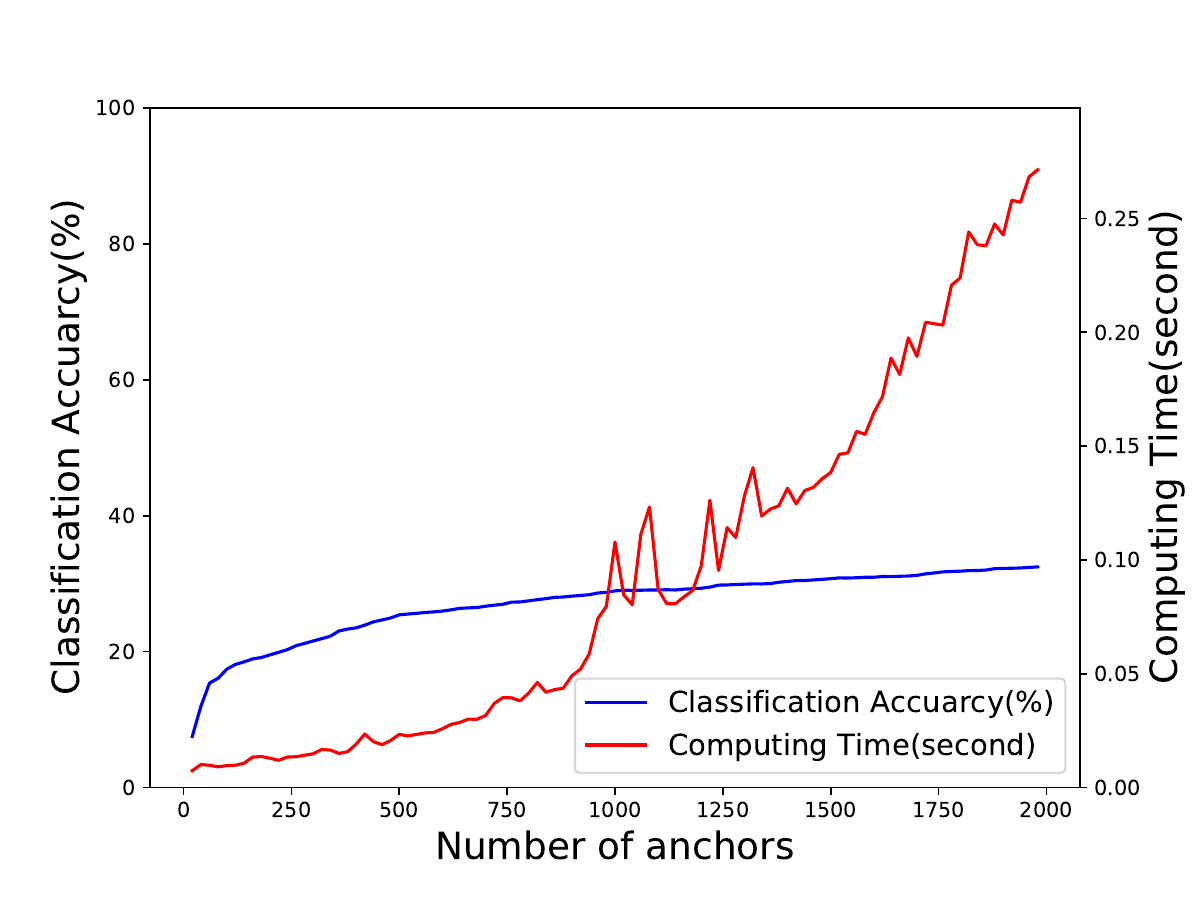}
}
\subfigure[USPS]{
\includegraphics[width=0.3\textwidth]{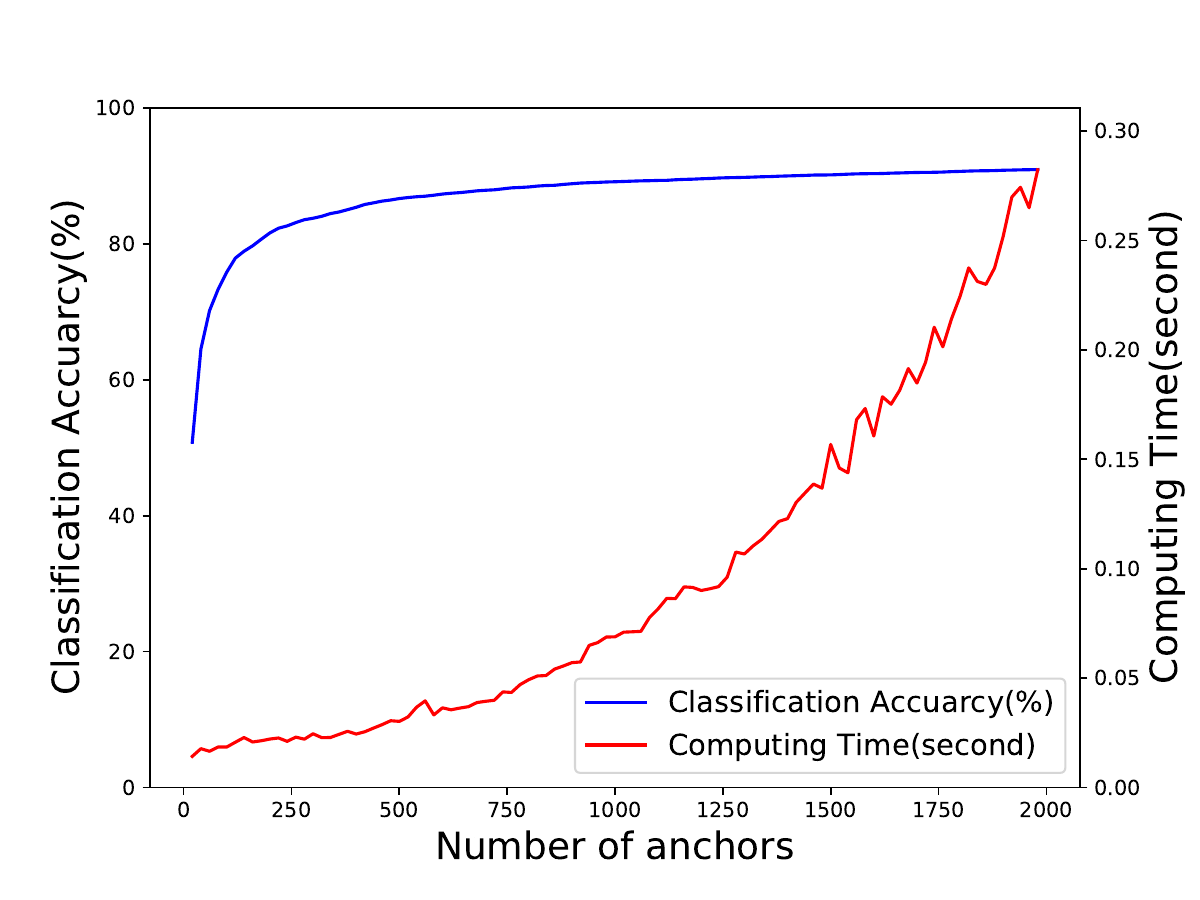}
}
\subfigure[CsMap]{
\includegraphics[width=0.3\textwidth]{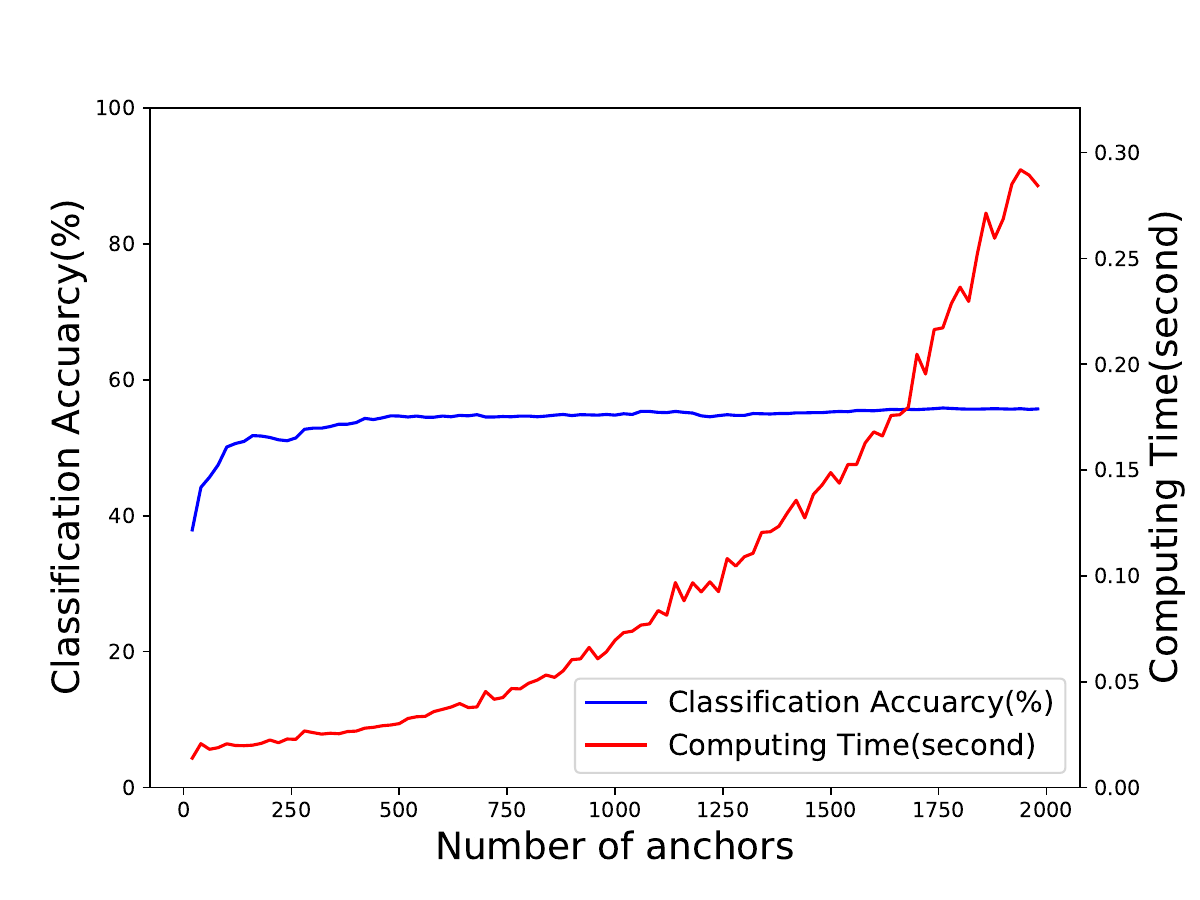}
}
\subfigure[PhishingWeb]{
\includegraphics[width=0.3\textwidth]{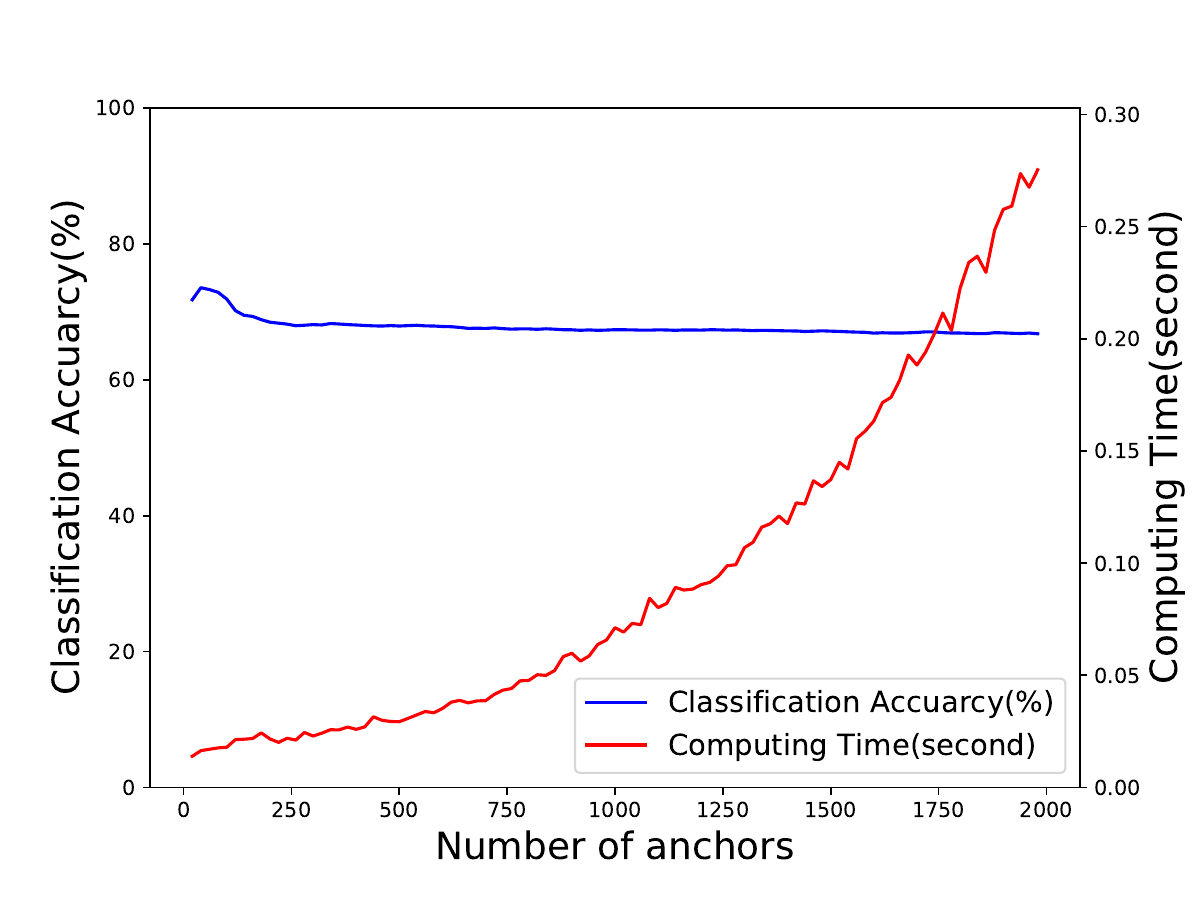}
}
\subfigure[Swarm]{
\includegraphics[width=0.3\textwidth]{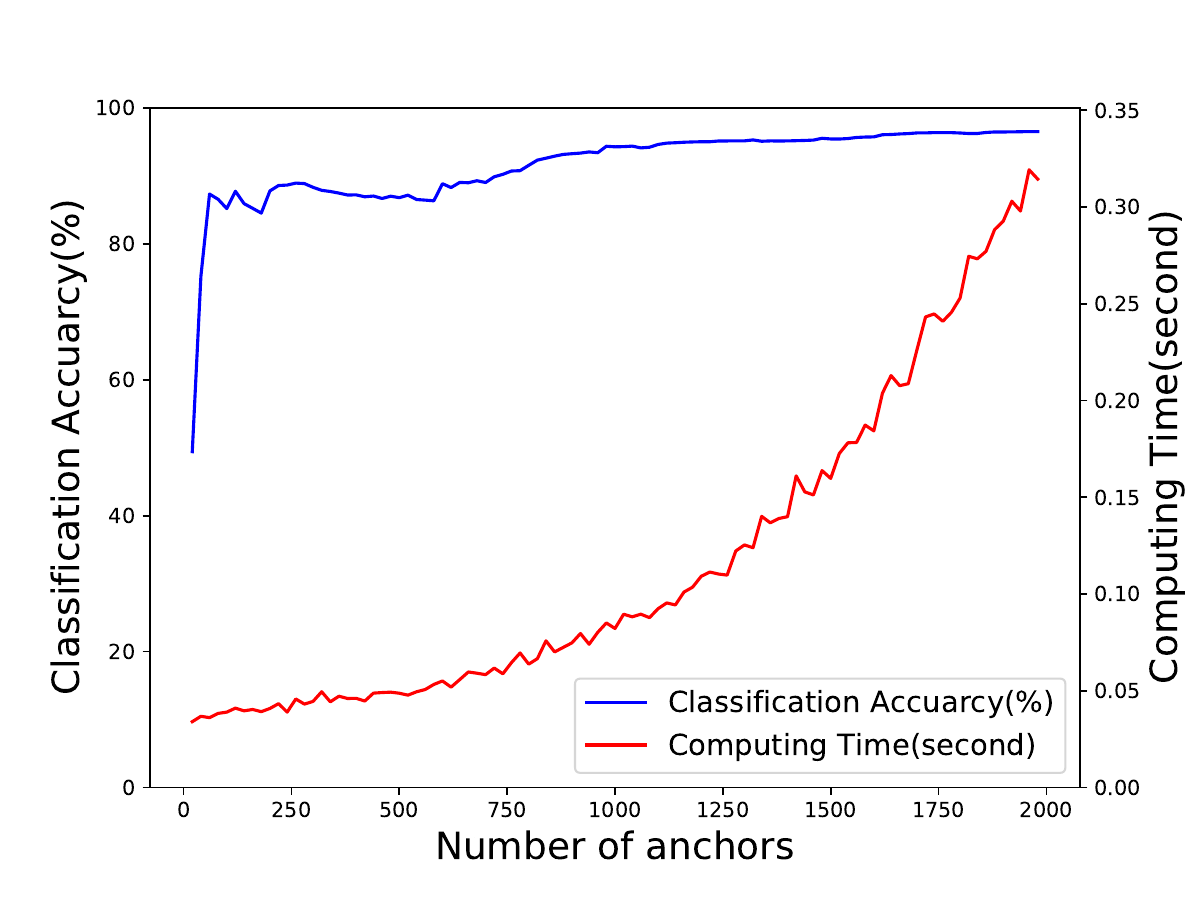}
}
\caption{Average classification accuracy and time cost versus the number of anchor points.}
\label{fig:acc-anchor}
\end{figure*}

\begin{figure*}[t]
\centering
\subfigure[MNIST]{
\includegraphics[width=0.3\textwidth]{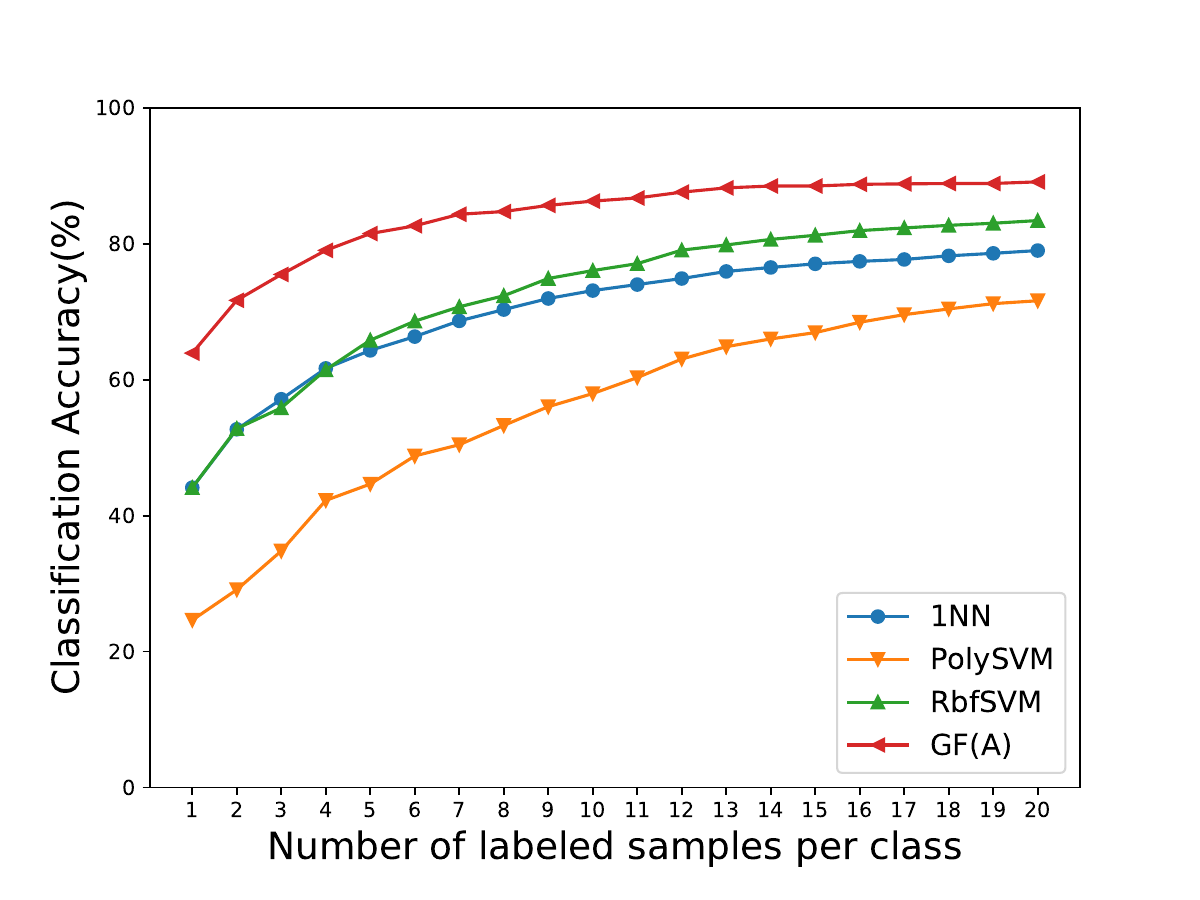}
}
\subfigure[USPS]{
\includegraphics[width=0.3\textwidth]{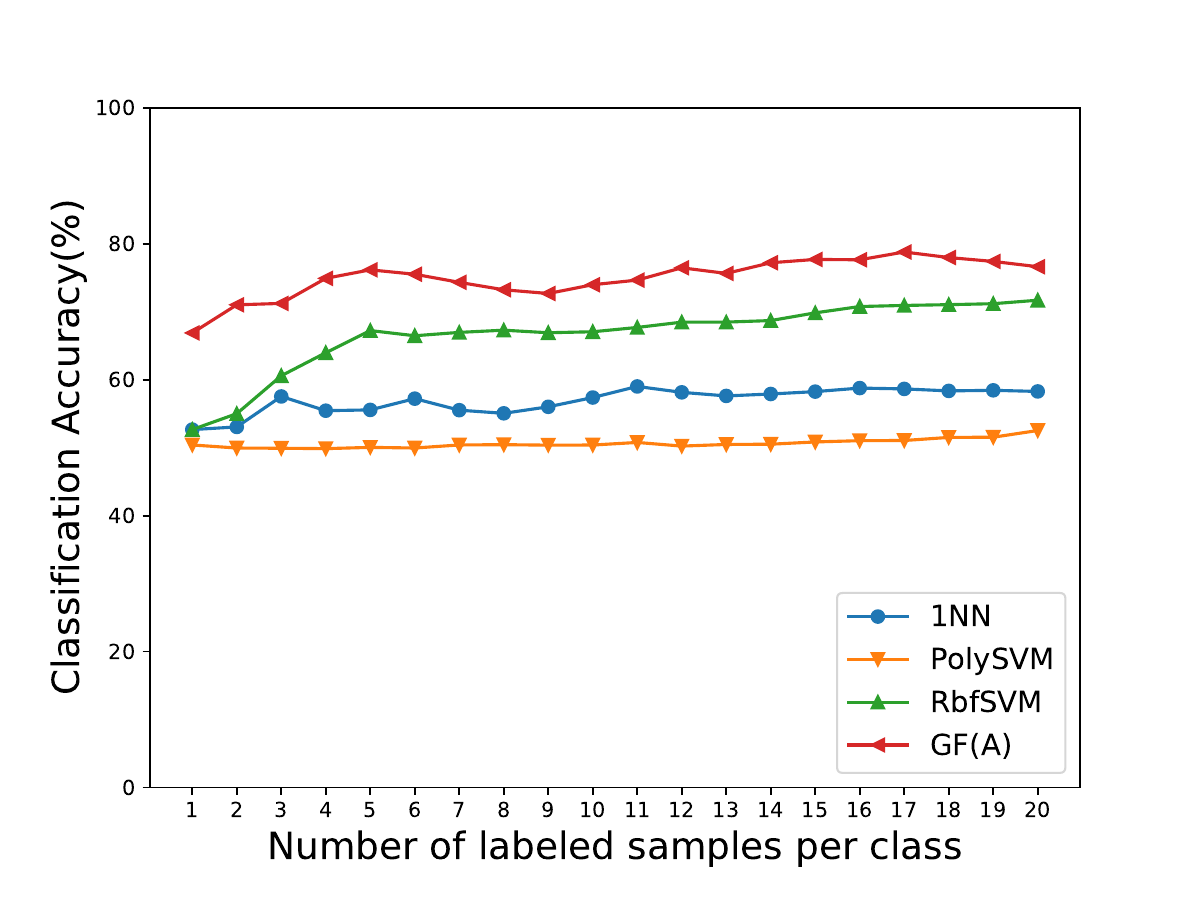}
}
\subfigure[EMNIST]{
\includegraphics[width=0.3\textwidth]{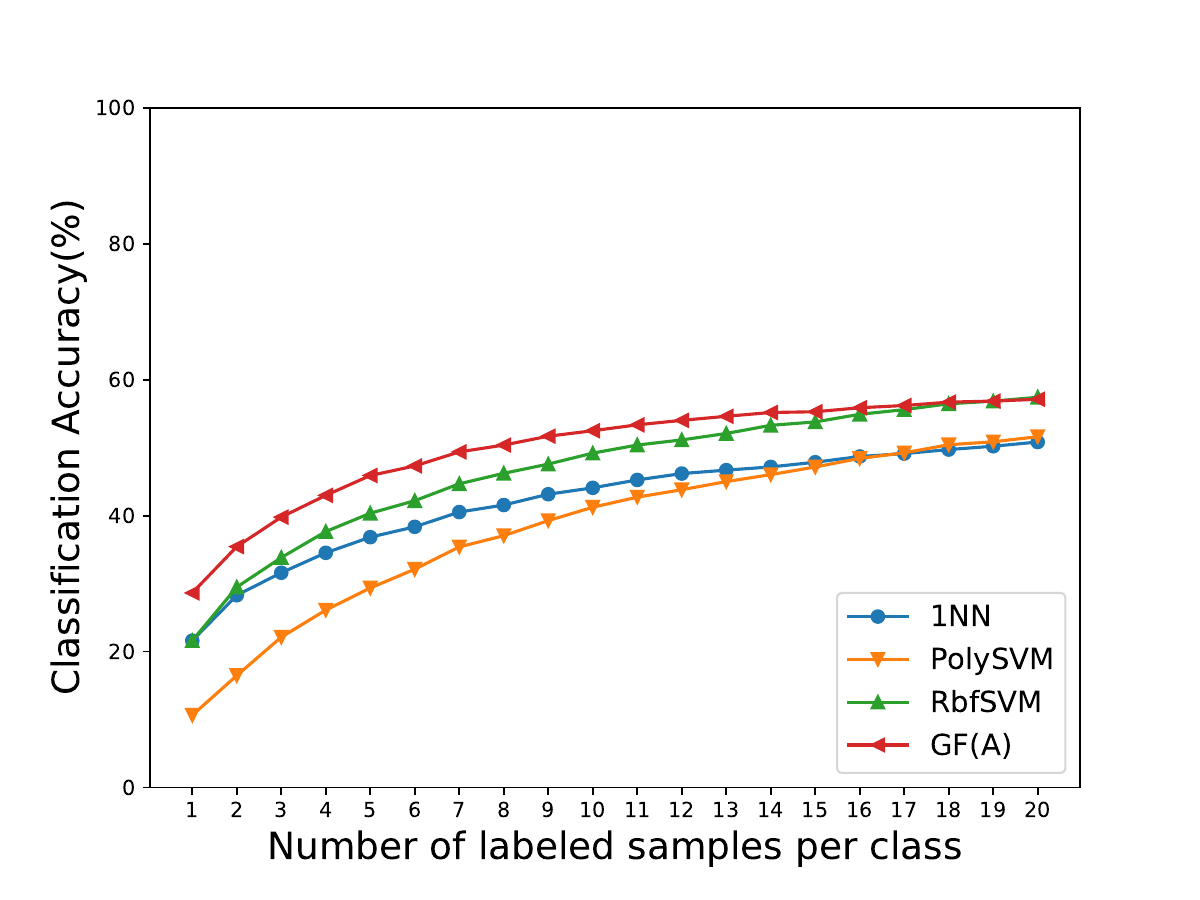}
}
\caption{Average classification accuracy versus the number of labeled samples per class.}
\label{fig:acc-label-large}
\end{figure*}

\noindent For an anchor-based method like GF(A), the selection of the number of anchor points may affect the effect. As is shown in Fig. \ref{fig:acc-anchor}, curves of the average accuracy and the computing time are drawn. The blue curves describe the trend of average accuracy versus the number of anchor points and the red curves describe that of average computing time versus the same. There are 5 trials for each point to prevent randomness due to the selection of 10 labeled samples for each class.

The time of constructing anchor-based graphs is not included, so the time complexity is strictly $O(nm^2)$. There's an obvious quadratic relation between the computing time and the number of anchors, which is consistent with our theoretical time complexity.  

When the number of anchors increases, performance increases rapidly initially and slower or even decreases later. Anchor points are used to represent the whole distribution and the number of classes usually affects the number of clusters in the distribution. Verifying this, the best performance always appears when the number of anchors is commensurate with that of classes. For the smallest-sized dataset Balance with 3 classes, fewer anchor points can represent all the samples so the best is achieved when there are as fewer anchor points as possible. At the same time, the number has to be larger than the nearest neighbor in our constructing graphs. So we set the number as 21. For datasets with 2 classes like Phishing Web, it is validated that using about 50 anchor points is a good choice. For datasets with more classes like USPS, CsMap, and Swarm, more anchor points are needed. We choose 1024 as the number of anchor points in the above-mentioned experiments due to its convenience for the recursion of BKHK. For datasets with 56 classes like MobileKSD, we use as many anchor points as possible.

\subsection{Experiments on Large Graphs.}
\noindent When dealing with large-sized datasets, GF, GF(G), and other semi-supervised algorithms(LLGC, HF) encounter out-of-memory errors while GF(A) achieves good performance. To evaluate GF(A), we introduce 3 larger-sized real-world datasets:

1) MNIST\cite{LeCun2005TheMD}. It is a digit dataset that contains 70,000 28×28 pixel grayscale samples of handwritten digits. It covers 10 classes and each class has 7,000 images.

2) SensIT\cite{Kaggle2023}. It is a dataset that contains data from sensors installed in vehicles and the data includes information about the vehicle’s speed, acceleration, and location. It’s a large-sized dataset containing 98,528 100-dimensional samples which need to be determined whether a fault exists.

3) EMNIST\cite{Cohen2017EMNISTEM}. It contains 814,255 handwritten characters converted to 28 × 28 pixel image format. We use the data organization of “bymerge” because this form of the dataset contains all the samples. Some letters of which the lowercase and uppercase are so similar that they are respectively merged into one class. After that, there are 47 classes in it.

We choose three classical supervised algorithms to compare with GF(A):

1) 1NN. With the nearest neighbor method, each unlabeled sample tends to have the label the same as that of the nearest labeled sample.

2) PolySVM\cite{Suykens1999LeastSS}. We use polynomial function as the kernel in the Support Vector Machine and employ the one-against-the-rest strategy in multi-class classification.

3) RbfSVM. It is the same as PolySVM but uses radial basis function as the kernel instead.

\begin{table}[t]
\centering
\caption{The Description of Three Large Datasets} \label{t:dataset-large}
\begin{center}
\begin{tabular}{ccccc}
\hline
\t{Dataset} & \t{Samples} & \t{Classes} & \t{Dimensions} & \t{Anchors} \\ 
\hline
MNIST & 70,000 & 10 & 784 & 1,024\\
SensIT & 98,528 & 2 & 100 & 50\\
EMNIST & 814,255 & 47 & 784 & 2,048 \\ \hline
\end{tabular}
\end{center}
\end{table}

\begin{table}[t]
\centering
\caption{Accuracy(\%) ± Standard Deviation(\%) of Different Approaches on Six Datasets} \label{t:acc-large}
\begin{center}
\begin{tabular}{ccccc}
\hline
\t{Dataset}&\t{1NN}&\t{PolySVM}&\t{RbfSVM}&\t{GF(A)}\\ \hline
MNIST&73.13±0.61&57.97±3.00&76.07±1.17&\t{86.30±1.09}\\
SensIT&57.39±4.77&50.40±0.55&67.07±3.20&\t{74.30±4.88}\\
EMNIST&44.10±1.32&41.24±1.02&49.21±1.15&\t{50.25±1.66}\\\hline
\end{tabular}
\end{center}
\end{table}

\begin{table}[t]
\centering
\caption{F1-macro of Different Approaches on Six Datasets} \label{t:F1-large}
\begin{center}
\begin{tabular}{ccccc}
\hline
\t{Dataset}&\t{1NN}&\t{PolySVM}&\t{RbfSVM}&\t{GF(A)}\\ \hline
MNIST&0.7348&0.6518&0.7636&\t{0.8622}\\
SensIT&0.6031&0.5592&0.6894&\t{0.7472}\\
EMNIST&0.4078&0.4128&0.4573&\t{0.4583}\\\hline
\end{tabular}
\end{center}
\end{table}

The conditions are the same as those in Section \ref{ex:acc}. According to the law mentioned in Section \ref{ex:anchors}, the number of anchor points for GF(A) are shown in TABLE \ref{t:dataset-large}. As shown in TABLE \ref{t:acc-large}, \ref{t:F1-large} and Fig. \ref{fig:acc-label-large}, GF(A) performs better compared with these supervised methods when used on large-sized datasets.

\section{Conclusions}
\noindent The Green-function method is a classical method in graph semi-supervised learning. However, its explanation is always the analogy as a whole and lacks interpretability from the perspective of optimization. We give a novel interpretation and the physical meanings by theoretical analysis in the former part of the paper. To avoid the disastrous consequence when using the Green-function method on non-fully graphs, an adjustment is applied to it and we propose the improved Green-function method. Based on the theoretical derivation, we could apply Gauss Elimination or Anchored Graphs to accelerate and decrease the space needed. At last, the extensive experiments prove our conclusions and the efficiency, accuracy, and stability of our proposed approach compared with the original Green-function method and other classified approaches.

\bibliographystyle{IEEEtran}
\bibliography{citation}

\begin{IEEEbiography}[{\includegraphics[width=1in,height=1.25in,clip,keepaspectratio]{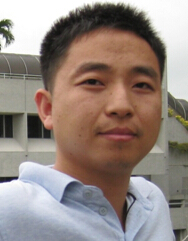}}]{Feiping Nie}
Feiping Nie received the Ph.D. degree in Computer Science from Tsinghua University, China in 2009, and is currently a full professor in Northwestern Polytechnical University, China. His research interests are machine learning and its applications, such as pattern recognition, data mining, computer vision, image processing and information retrieval. He has published more than 100 papers in the following journals and conferences: TPAMI, IJCV, TIP, TNNLS, TKDE, ICML, NIPS, KDD, IJCAI, AAAI, ICCV, CVPR, ACM MM. His papers have been cited more than 20000 times and the H-index is 99. He is now serving as Associate Editor or PC member for several prestigious journals and conferences in the related fields.
\end{IEEEbiography}

\begin{IEEEbiography}[{\includegraphics[width=1in,height=1.25in,clip,keepaspectratio]{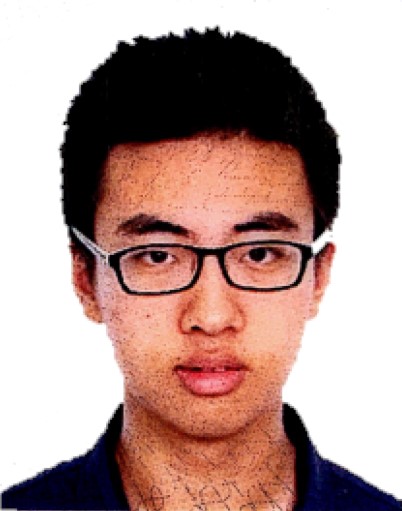}}]{Yitao Song}
Yitao Song received the B.S. from Northwestern Polytechnical University, Xi’an, China, in 2020. He is currently pursuing the Ph.D. degree with the School of Computer Science and the School of Artificial Intelligence, Optics and Electronics (iOPEN), Northwestern Polytechnical University, Xi’an, China.
His research interests include machine learning and data mining.
\end{IEEEbiography}

\begin{IEEEbiographynophoto}{Wei Chang}
Wei Chang received the BE degree in statistics from Northwestern Polytechnical University, Xi’an, China, in 2017, where he is currently working toward the master’s degree in the Center for Optical Imagery Analysis and Learning, Northwestern Polytechnical University, Xi’an, China. 
His research interests include machine learning, computer vision, and pattern recognition.
\end{IEEEbiographynophoto}

\begin{IEEEbiography}[{\includegraphics[width=1in,height=1.25in,clip,keepaspectratio]{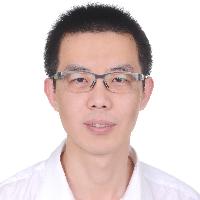}}]{Rong Wang}
Rong Wang is with the School of Artificial Intelligence, OPtics and ElectroNics (iOPEN), Northwestern Polytechnical University, Xi’an 710072, P.R. China, and also with the Key Laboratory of Intelligent Interaction and Applications (Northwestern Polytechnical University), Ministry of Industry and Information Technology, Xi’an 710072, P.R. China.
\end{IEEEbiography}

\begin{IEEEbiographynophoto}{Xuelong Li}
(Fellow, IEEE) is currently a Full Professor with the School of Artificial Intelligence, Optics and Electronics (iOPEN), Northwestern Polytechnical University, Xi’an, China.
\end{IEEEbiographynophoto}

\end{document}